\documentclass[letterpaper]{article} 
\usepackage[submission]{aaai2026}  
\usepackage{times}  
\usepackage{helvet}  
\usepackage{courier}  
\usepackage[hyphens]{url}  
\usepackage{graphicx} 
\usepackage{natbib}  
\usepackage{caption} 
\usepackage{addition}
\allowdisplaybreaks
\usepackage{newfloat}
\usepackage{listings}
\DeclareCaptionStyle{ruled}{labelfont=normalfont,labelsep=colon,strut=off} 
\floatstyle{ruled}
\newfloat{listing}{tb}{lst}{}
\floatname{listing}{Listing}
\title{Near-Optimal Regret for Efficient Stochastic Combinatorial Semi-Bandits}
\author {
    Zichun Ye\textsuperscript{\rm 1},
    Runqi Wang\textsuperscript{\rm 1},
    Xutong Liu\textsuperscript{\rm 2},
    Shuai Li \textsuperscript{\rm 1}
}
\affiliations {
    \textsuperscript{\rm 1}Shanghai Jiao Tong University\\
    \textsuperscript{\rm 2}University of Washington \\
    \{alchemist, allen-w, shuaili8\}@sjtu.edu.cn, xutongl@uw.edu
}

\usepackage{bibentry}

\begin{document}

\maketitle

\begin{abstract}
The combinatorial multi-armed bandit (CMAB) is a cornerstone of sequential decision-making framework, dominated by two algorithmic families: UCB-based and adversarial methods such as follow the regularized leader (FTRL) and online mirror descent (OMD). However, prominent UCB-based approaches like CUCB suffer from additional regret factor $\log T$ that is detrimental over long horizons, while adversarial methods such as EXP3.M and HYBRID impose significant computational overhead. To resolve this trade-off, we introduce the Combinatorial Minimax Optimal Strategy in the Stochastic setting (CMOSS). CMOSS is a computationally efficient algorithm that achieves an instance-independent regret of $O\big( (\log k)\sqrt{kmT}\big )$ when $k\leq \frac{m}{2}$ and $O\big((m-k)\sqrt{\log k\log(m-k)T}\big )$ when $k>\frac{m}{2}$ under semi-bandit feedback, where $m$ is the number of arms and $k$ is the maximum cardinality of a feasible action. Crucially, this result eliminates the dependency on $\log T$ and matches the established lower bounds of $\Omega\big(\sqrt{kmT}\big)$ when $k\leq \frac{m}{2}$ and $\Omega\big((m-k)\sqrt{\log (\frac{m}{m-k}) T}\big)$ when $k>\frac{m}{2}$ up to logarithmic terms of $k$ and $m$. We then extend our analysis to show that CMOSS is also applicable to cascading feedback.
Experiments on synthetic and real-world datasets validate that CMOSS consistently outperforms benchmark algorithms in both regret and runtime efficiency.
\end{abstract}


\section{Introduction} \label{sec:intro}

The stochastic multi-armed bandit (MAB) is a fundamental $T$-round sequential decision-making game that has been widely studied in the past two decades \cite{AB09,ACFS95,ACFS02,Rob52}.
In this game, one player faces $m$ arms. In each round $t \in [T]$, the player chooses an arm $A_t$ among the $m$ arms and gets a reward drawn from a predetermined but unknown reward distribution with mean $\mu_{A_t} \in [0,1]$. The arm with the largest reward mean $\mu_*$ is called the best arm. The total expected regret is defined as $\E{R(T)} = \E{\sum_{t=1}^T \tp{\mu_* - \mu_{A_t}}},$ where the expectation is over the randomness from the player's strategy and the random reward. And the objective is to minimize the minimax regret, namely the regret of the best algorithm against the worst input.

The Upper Confidence Bounds (UCB) algorithm \cite{ACBF02} is a well-known approach for addressing the stochastic MAB problem, achieving an expected regret of $O\tp{\sqrt{mT\log T}}$. Building on this foundation, its variant, the Minimax Optimal Strategy in the Stochastic case (MOSS) algorithm \cite{AB09,DP16,MG17}, improves the regret bound to $O\tp{\sqrt{mT}}$ by refining the upper confidence bounds and employing a distinct regret analysis. These results align with the minimax lower bound of $\Omega\tp{\sqrt{mT}}$ established in \cite{ACFS02}, demonstrating their optimality in the stochastic setting.
In parallel, the EXP3 algorithm \cite{ACFS95,LW94} provides a solution for adversarial environments. EXP3 is a classic implementation of the follow the regularized leader (FTRL) framework using a negative entropy regularizer \cite{SS12}.
This approach, also known through the dual perspective of online mirror descent (OMD), achieves a regret of $O\left(\sqrt{mT \log m}\right)$ in both adversarial and stochastic regimes, notably avoiding dependence on $\log T$.

In many real-world applications, however, decision-making problems often extend beyond the simple MAB framework to involve \textit{combinatorial actions} among multiple arms. For instance, in online advertising \cite{KSWA15}, an advertiser could select up to $k$ web pages to display ads, constrained by a budget. Each user visits a subset of pages, and the advertiser must learn the unknown click-through probabilities to maximize clicks. Similarly, in viral marketing \cite{GKJ12,TZH17}, a marketer selects seed nodes in a social network to initiate information cascades, learning influence probabilities to maximize the reach of campaigns. 
These scenarios highlight the need for algorithms that efficiently address the combinatorial MAB (CMAB) problem. 

While a number of works \cite{CTPL15,KWAS15,CHL16,WC17,LZWJ22,LWZ24,LDWH25} have focused on the CMAB problem, existing algorithms either suffer from additional $\log T$ terms in their regret bounds or incur high computational costs in practice. Among them, the Combinatorial UCB (CUCB) algorithm \cite{CWY13,KWAS15,WC17,LZWJ22} is one of the most prominent and widely adopted approaches. The CUCB algorithm extends the classic UCB framework by selecting the action (super arm) from a feasible set that solves a combinatorial optimization problem defined over the upper confidence bounds of the base arms. Kveton et al. \cite{KWAS15} establish an upper regret bound of $O\left(\sqrt{kmT \log T}\right)$ for CUCB, which matches the lower bound of $\Omega\left(\sqrt{kmT}\right)$ only up to a $\sqrt{\log T}$ factor when $k\leq \frac{m}{2}$. 

A separate line of work \cite{CBL12,CTPL15,UNK10,ZLW19} focuses on adapting algorithms derived from FTRL and OMD frameworks that could be applied to CMAB setting. The best-known result in this line \cite{ZLW19} achieves a regret of $O\left(\sqrt{kmT} \right)$ when $k \leq \frac{m}{2}$ and $O\left((m-k)\sqrt{\log\tp{\frac{m}{m-k}} T} \right)$ when $k> \frac{m}{2}$, successfully removing the $\log T$ term. However, these adversarially designed algorithms are often computationally expensive, 
limiting their practical utility. This gap—between computational efficiency and regret optimality—has remained unresolved for over a decade, highlighting the critical need for ``MOSS-like'' algorithms in CMAB that achieve minimax-optimal regret without $\log T$ dependence.




\subsection{Our results}

To address the above limitation, we study the stochastic combinatorial semi-bandit problem and propose the CMOSS algorithm. Our contribution is to show that CMOSS is not only computationally efficient but also has a regret of $O\big( (\log k)\sqrt{kmT}\big )$ when $k\leq \frac{m}{2}$ and $O\big((m-k)\sqrt{\log k\log(m-k)T}\big )$ when $k>\frac{m}{2}$. We also extend the setting to cascading feedback and compare our result with CUCB, since the other three algorithms can not be applied to the scenarios where arms in $A_t$ are probabilistically triggered. The results are summarized in \Cref{tab:theory_results}.

\begin{table*}[t] 
    \caption{Regret bounds and empirical computational cost per-round for algorithms on the CMAB problem. 
    \\ $\star$ COMBEXP is from \cite{CTPL15}. CUCB is from \cite{CWY13,KWAS15,LWZ24}. EXP3.M is from \cite{UNK10}. HYBRID is from \cite{ZLW19}. The detailed regret bound of CMOSS is given in \Cref{thm:cmoss_ub} and \Cref{thm:cmoss_ub_cascading}.
    \\ $\circ$ ``None'' means these methods can not deal with this case.   $p^*$ denotes the minimal probability that an action could have all base arms observed over all time. 
    \\ $\dagger$ We select $k=20, m=30$ for HYBRID and CMOSS when $k>\frac{m}{2}$ and $k=10, m=30$ for other rows with $T=100000$ separately for runtime measurement on a synthetic dataset (averaged over 10 runs) in the semi-bandit feedback setting. See \Cref{sec:comparison_with_baselines} for dataset details. 
    \\ $\ddagger$  HYBRID performs relatively bad in experiments simulating realistic conditions where arm means are in the range of $[0,0.1]$, although the regret bound is optimal. See \Cref{sec:comparison_with_baselines} for dataset details.} 
    \renewcommand{\arraystretch}{2}
    \begin{center}
       \begin{tabularx}{\textwidth}{lccc}
            \toprule
            \textbf{Algorithm\textsuperscript{$\star$}}  & \textbf{Semi-Bandit Feedback} & \textbf{Cascading Feedback\textsuperscript{$\circ$}}   & \makecell{\textbf{Computation Cost} \\ \textbf{Per-Round (Average,} \\ \boldmath$10^{-4}$ \textbf{sec)} \textsuperscript{$\dagger$}}  \\
            \midrule
            COMBEXP    & $O\left(\sqrt{k^3mT\log m} \right)$ & None  & Too large \\
            \hline
            CUCB   & $O\left(\sqrt{kmT\log T} \right)$ & $O\left(\sqrt{kmT\log T} \right)$ &  1.516
            \\
            \hline
            EXP3.M    & $O\left(\sqrt{kmT\log \left(\frac{m}{k} \right)} 
            \right)$ &  None   & 10.429 
            \\
            \hline
            \rule{0pt}{15pt} \multirow{2}{*}{HYBRID\textsuperscript{$\ddagger$}} 
                & $O\left(\sqrt{kmT} \right)$, $k \leq \frac{m}{2}$ &  None   &  12.112 \\
            \cline{2-4}
            \rule{0pt}{20pt} ~ & $O\left((m-k)\sqrt{\log\left(\frac{m}{m-k} \right) T} \right)$, $k> \frac{m}{2}$  & None & 11.195 \\
            \hline
            \rule{0pt}{15pt} \multirow{2}{*}{CMOSS} 
                & $O\left((\log k)\sqrt{kmT} \right)$, $k \leq \frac{m}{2}$ &  $O\tp{\frac{\log k}{p^*}\sqrt{kmT}}$   &  1.737 \\
            \cline{2-4}
            \rule{0pt}{20pt} ~ & $O\left((m-k)\sqrt{\log k\log(m-k) T} \right)$, $k> \frac{m}{2}$  & $O\left( \frac{m-k}{p^*}\sqrt{\log k\log(m-k) T} \right)$  &  3.067 \\
            \bottomrule
        \end{tabularx}
    \end{center}
    \label{tab:theory_results}
\end{table*}

The $\log T$ factor of many UCB-style algorithms is due to ensuring that the empirical mean of every arm must lie within its confidence interval. To achieve optimal instance-independent regret bound, the MOSS algorithm bypasses this requirement. MOSS operates by classifying arms and then directly bounding the expected plays of suboptimal arms with relatively large gaps, using modified confidence intervals to facilitate this analysis. However, this is not directly tractable in the CMAB setting. First, how to classify the suboptimal actions based on their gaps; second, how to bound the expected number of times each suboptimal action is pulled; and finally, how to tractably aggregate the regret contributions from a potentially exponential number of such actions.

Our technical contribution lies in bridging the gap between the MOSS-style analysis and the combinatorial structure of CMAB, which prior work such as CUCB does not address. To this end, we introduce new analytical tools that enable regret-optimal and computationally efficient learning. Our key innovations include:

\begin{enumerate}[leftmargin=*, align=left] 



    \item \textbf{Targeted action filtering.} We identify suboptimal actions with gaps exceeding the worst-case underestimation error of the optimal action. Actions with gaps smaller than a chosen threshold $\Delta$ contribute at most $T\Delta$ to the regret, allowing us to focus on a tractable subset of actions. 

    \item \textbf{Constraint decoupling via classification.} In combinatorial settings, each action involves multiple arms, which complicates concentration analysis. To address this, we introduce a partitioning scheme that classifies suboptimal actions based on structural properties. This enables us to decouple constraints at the arm level and apply refined concentration bounds. 

    \item \textbf{Regret decomposition.} After the preparation of the first two points, we need to choose a proper way to calculate the regret to avoid redundant computations. Since each suboptimal arm may participate in an exponential number of actions, directly summing over all possibilities is infeasible. We instead decompose the total regret into per-arm contributions and use an integral-based approximation to tightly and efficiently estimate the overall impact. See \Cref{sec:cmoss_proof} for more detailed analysis.

\end{enumerate}
Furthermore, we show that CMOSS could also be extended to cascading feedback. More details are given in \Cref{sec:cmoss_cascading}. We also run comparative experiments evaluating the performance of CUCB and CMOSS in \Cref{append:cascading}.

We also conduct comparative experiments involving CMOSS, CUCB, EXP3.M, and HYBRID on both synthetic and real-world datasets under semi-bandit feedback, including ablation studies. Detailed descriptions of CUCB, EXP3.M, and HYBRID are provided in \Cref{append:algo} for completeness. The results demonstrate that, under noisy conditions and after sufficient exploration rounds, CMOSS incurs slightly higher computational cost than CUCB while achieving notably lower cumulative regret. Compared to EXP3.M and HYBRID, CMOSS attains better cumulative regret with significantly reduced computational cost. This performance profile highlights CMOSS’s superior balance between exploration and exploitation, underscoring its effectiveness in such environments. Further details are provided in \Cref{sec:experiments}.

\subsection{Related work}

\textbf{Multi-Armed Bandits (MAB).} The MAB problem was first introduced in \cite{Rob52}. Audibert et al. \cite{AB09} later established an optimal regret bound of $\Theta(\sqrt{mT})$ for both stochastic and adversarial settings. Our work focuses on the stochastic setting, where the UCB algorithm \cite{ACBF02} plays a pivotal role and serves as the foundation for many CMAB studies. To eliminate the $\log T$ factor, the MOSS algorithm was proposed in \cite{AB09}, with an anytime modification introduced in \cite{DP16}, achieving a regret of $O\tp{\sqrt{mT}}$. In contrast, works such as \cite{ACFS95,LW94} targeted the adversarial setting and proposed the EXP3 algorithm, which also eliminates the $\log T$ factor and achieves a regret of $O\tp{\sqrt{mT\log m}}$. However, EXP3 incurs significantly higher computational overhead compared to UCB-based methods such as MOSS. Our work builds primarily on MOSS and extends it to the combinatorial setting.

\textbf{Stochastic Combinatorial Multi-Armed Bandits (CMAB).} The stochastic CMAB problem has been extensively studied in the literature \cite{CHL16,CWY13,CTPL15,KSWA15,KWAS15,WC17,LZWJ22,LZWL23}. The seminal work by Gai et al. \cite{GKJ12} pioneered the study of stochastic CMAB with semi-bandit feedback, which inspired algorithmic advancements and improved regret bound across various settings \cite{KSWA15,KWAEE14,TZH17}. For the instance-dependent regret, Wang et al. \cite{WC18} proposed the combinatorial Thompson sampling (CTS) algorithm with an offline oracle, achieving a regret of $O\tp{\frac{m\log T}{\Delta_{\min}}}$, where $\Delta_{min}$ is the minimum gap between the expected reward of the optimal action and any suboptimal action. The bound essentially matches asymptotically on $T$. Thus in this paper, we focus on the instance-independent regret.

For the instance-independent regret, one line of works focused on the adversarial setting and designed algorithms based on the principle of OMD and FTRL. Several studies \cite{CBL12,DKH07,UNK10,ZLW19} proposed enhancements such as COMBEXP, EXP3.M, with the best \cite{ZLW19} achieving a regret of $O\tp{\sqrt{kmT}}$ when $k \leq \frac{m}{2}$ and $O\left((m-k)\sqrt{\log\tp{\frac{m}{m-k}} T} \right)$ when $k> \frac{m}{2}$. The latter case matches the lower bound of $\Omega\left( (m-k)\sqrt{\log\tp{\frac{m}{m-k}} T} \right)$. However, these methods suffer from high computational overhead. Another line of works \cite{CWY13,LZWJ22,WC17,WC18,LZWL23} used UCB-based methods. Chen et al. \cite{CWY13} proposed the CUCB algorithm and analyzed the regret of CUCB with an approximation oracle. Then Kveton et al. \cite{KWAS15} derived an upper bound of $O\tp{\sqrt{kmT\log T}}$ for CUCB and established a lower bound of $\Omega\tp{\sqrt{kmT}}$ when $k\leq \frac{m}{2}$. To address the degradation caused by the $\log T$ term over long horizons, we refine the upper confidence bound and propose the CMOSS algorithm. It achieves strong performance while maintaining low computational costs, even after removing the $\log T$ term.

\section{Preliminaries}

We begin by introducing the general framework of the stochastic CMAB and establishing notations used throughout the paper. The CMAB problem is modeled as a sequential decision-making game between a learning agent and an environment. The environment consists of $m$ base arms $\arm_1, \ldots, \arm_m$, each associated with a random reward $X_i \sim D_i$ supported on $[0,1]$ and having mean $\mu_i$. The distributions $\{D_i\}_{i=1}^m$ are fixed by the environment before the game begins and are drawn from a known class $\mathcal{D}$. The learner knows $\mathcal{D}$ but not the specific instantiation of each $D_i$. In each round, the rewards of these arms may be correlated with each other.
		
The action space $\+A$ consists of all subsets of $[m]$ such that $\forall A\in \+A$, $\abs{A} \leq k$. The $T$-round decision game starts as follows: in each round $t\in[T]$, the player selects an action $A_t$ from $\+A$ based on the feedback history, and the environment draws from every distribution $D_i$ a sample $X^{(t)}_i$. Let $\vmu$ represents the vector of the mean rewards of all arms. After the action $A_t$ is played, under semi-bandit feedback, the values of $X^{(t)}_i$ for all $i\in A_t$ are observed as the feedback to the player. Under cascading feedback, the arms in $A_t$ are probabilistically triggered and only the arms in the triggered set $\tau_t$ could be observed. In both cases the player receives $r(A_t, \vmu)$ as the reward in round $t$.

The player's objective is to minimize the expected difference between the cumulative reward of the best action $A^*$ and the player's own cumulative reward. That is, the player aims to design an algorithm to minimize the expected regret 
\begin{align*}
\E{R(T)} 
&= \E{\sum_{t=1}^T \tp{ r(A^*, \vmu) - r(A_t, \vmu) }}.
\end{align*}

\textbf{Notations.}
$\!U(a,b)$ represents the uniform distribution on $(a,b)$. Denote $\Delta_{A_t} = r(A^*, \vmu) - r(A_t, \vmu)$ as the gap of $A_t$. $\1{S}$ is the indicator variable of whether event $S$ happens. For every $i\in [m]$, we use a random variable $T_{i,t} \defeq \sum_{s=1}^t  \1{i \in A_s}$ to denote the number of times $\arm_i$ has been pulled when the $t$-th round ends. For two sets $A$ and $B$, $A\setminus B = \set{a\mid a\in A \land a \notin B}$. And $\log$ represents the logarithm with base $2$. $\lfloor x \rfloor$ represents the largest integer that is no more than $x$. $\ln^+(x)=\ln\max\set{1,x}$.

\section{CMOSS Algorithm and Its Regret Analysis} \label{sec:cmoss}

In this section, we first define CMOSS and present the theorem for its upper bound in \Cref{sec:cmoss_algo}. Next, in \Cref{sec:cmoss_proof}, we focus on the case of semi-bandit feedback, outline the proof idea and highlight key lemmas. Finally, in \Cref{sec:cmoss_cascading}, we extend our analysis to cascading feedback.

\subsection{The CMOSS algorithm} \label{sec:cmoss_algo}

\begin{algorithm}[H]
\caption{\textbf{C}ombinatorial \textbf{M}inimax \textbf{O}ptimal \textbf{S}trategy in the \textbf{S}tochastic setting (CMOSS)} 
\label{alg:cmoss}
\textbf{Input}:{Number of base arms $m$, cardinality $k$, parameter $\delta$.}
    \begin{algorithmic}[1]
        \State For each arm $i$, $T_{i}\leftarrow 0$
        \Comment{maintain the total number of times arm $i$ is played so far}
        \State For each arm $i$, $\hat \mu_{i}\leftarrow 1$
        \Comment{maintain the empirical mean of $\arm_i$}
        \For{$t=1, 2, 3, \ldots$}
            \State For each arm $i\in [m]$, $\rho_i \leftarrow \sqrt{\frac{1}{T_i}\ln^+ \tp{\frac{1}{\delta T_i}}}$
            \Comment{the confidence radius, $\rho_i=+\infty$ if $T_i=0$}
            \State For each arm $i\in [m]$, $\bar{\mu}_{i} \leftarrow  \min\left \{\hat{\mu}_{i} + \rho_i,1\right \}, \bar\vmu \leftarrow (\bar{\mu}_{1}, \cdots, \bar{\mu}_{m})$   
            \State $A_t \leftarrow \argmax_{A\in \+A} r(A, \bar\vmu)$
            \State Pull $A_t$ of base arms and obtain the triggered set $\tau_t$. Observe feedback $X_i^{(t)}$ for all $i\in \tau_t$
            \State For every $i \in \tau_t$, update $T_i$ and $\hat{\mu}_{i}$: 
            $T_i \leftarrow T_i + 1$, $\hat{\mu}_{i} \leftarrow \hat{\mu}_{i} + (X_i^{(t)}-\hat{\mu}_{i})/T_i$
        \EndFor
    \end{algorithmic}
\end{algorithm}

The CMOSS algorithm maintains the empirical estimate $\hat \mu_i$ for the true mean $\mu_i$ for all the $m$ arms, then selects the action by solving the combinatorial problem of maximizing $r(A, \bar \vmu)$, where each $\bar \mu_i$ optimistically estimates $\mu_i$ by a confidence interval $\rho_i$. The prominent difference between CMOSS and CUCB is that the radius $\rho_{i} = \sqrt{\frac{1}{T_i} \ln^+ \tp{\frac{1}{\delta T_i}}}$ but not $\sqrt{\frac{3\ln t}{2T_i}}$, where $\delta$ is a parameter to be optimized during the proof. In this case, the probability that $\hat \mu_i$ lies in the confidence region $(\mu_i-\rho_i, \mu_i+\rho_i)$ in every round $t$ could be much less than $\Theta\tp{\frac{1}{t^2}}$. Thus we can not neglect the regret caused by $\hat \mu_i$ falling outside $(\mu_i-\rho_i, \mu_i+\rho_i)$. Therefore we need to bound the regret in the average sense. We have the following theorem:

\begin{theorem} \label{thm:cmoss_ub}
    For any stochastic combinatorial semi-bandit feedback instance with $m$ arms, time horizon $T$, and the maximum number of arms that could be chosen in a round is $k$, then when $k\leq \frac{m}{2}$, the expected regret of \Cref{alg:cmoss} is $O\tp{(\log k)\sqrt{kmT}}$ and when $k > \frac{m}{2}$, the expected regret is $O(\tp{m-k}\sqrt{\log k \log(m-k) T})$.
\end{theorem}

\Cref{thm:cmoss_ub} eliminates the $\sqrt{\log T}$ factor in the upper bound of CUCB. Unlike the way EXP3.M and HYBRID do, CMOSS also retains the low computation cost, introducing only an additional logarithmic term of $k$ and $m$. And when $k = m-c$ ($c$ is a constant), our upper bound strictly matches the lower bound, which is $\Theta\tp{c\sqrt{T\log m}}$.

\subsection{Regret analysis for semi-bandit feedback} \label{sec:cmoss_proof}

In this case $r(A_t, \vmu) = \sum_{i\in A_t} \mu_i$. Since $\+A$ consists of all subsets of $[m]$ with cardinality less than $k$, and the reward function is monotonically increasing, then $A^*$ consists of $k$ arms corresponding to the largest values of $\mu_i$. $A_t$ is the set of $k$ arms corresponding to the largest values of $\bar\mu_{i}$ in round $t$. Since $\Delta_{A_t} = \sum_{j\in A^*} \mu_j - \sum_{i\in A_t} \mu_i$, we have
\begin{align*}
    \E{R(T)} 
    &= \E{\sum_{t=1}^T \tp{\sum_{j\in A^*} \mu_j - \sum_{i\in A_t} \mu_i}} \\
    &= \E{\sum_{t=1}^T\Delta_{A_t}}.
\end{align*}

Denote $\hat \mu_{i,t}, \rho_{i,t}$ as the value of $\hat \mu_i$ and $\rho_i$ when the $t$-th round ends. 
There are three pivotal challenges that are solved in our analysis: 
\begin{enumerate}[leftmargin=*, align=left]
    \item Extracting the key actions to consider. We focus on bounding the regret for actions whose gap is larger than the summation of $\mu_j - \min\limits_{t\le T} \tp{\hat{\mu}_{j,t} + \rho_{j,t}}$ across all $\arm_j$ in $A^*$, since there are $k$ arms in $A^*$ and $A_t$ at each round $t$; 
    
    \item Decoupling constraints across suboptimal action $A_t$ into individual arms. We classify the action set into well-defined categories (\Cref{lem:actionset_divide}) and quantitatively characterize the necessary condition that must be satisfied by the suboptimal arms belong to $A_t$ in each category \Cref{lem:indicator_decompose});
    
    \item Calculating the regret. We decompose the overall regret into contributions from each suboptimal arm. To manage the exponential number of actions for every suboptimal arm, we construct an integral to bound their summation, which effectively reduces redundant computations. 
\end{enumerate}

In this section we only present the case where $k\leq \frac{m}{2}$. All proofs of the lemmas in this section and the other case are deferred to \Cref{append:cmoss_proof}. Here $\delta$ is optimized to be $\frac{m(\log k)^2}{kT}$. Define $\zeta_j = \max\tp{0, \mu_j - \min\limits_{t\le T} \tp{\hat{\mu}_{j,t} + \rho_{j,t}}}$ for every $j\in A^*$. We decompose the overall regret as follows:
\begin{align}
    &\phantomeq \sum_{t=1}^T\E{\Delta_{A_t}} \notag \\
    &\leq \sum_{t=1}^T\E{\1{\Delta_{A_t}\leq (\log k)\sqrt{\frac{ek m}{T}} }\Delta_{A_t}} +  \notag \\
    &\sum_{t=1}^T\E{\1{\Delta_{A_t}\leq 8\sum_{j\in A^*}\zeta_j}\Delta_{A_t}} \notag +\\
    &\sum_{t=1}^T\E{\1{\Delta_{A_t} > \max\set{8\sum_{j\in A^*}\zeta_j, (\log k)\sqrt{\frac{ek m}{T}}}} \Delta_{A_t}} \notag \\
    &\leq \sqrt{ek mT} + 8\sum_{j\in A^*} \E{\zeta_j}\cdot T + \notag \\
    &\sum_{t=1}^T\E{\1{\Delta_{A_t} > \max\set{8\sum_{j\in A^*}\zeta_j, (\log k)\sqrt{\frac{ekm}{T}}}} \Delta_{A_t}}. \label{line:decompose}
\end{align}

From \Cref{lem:moss-Delta} in \Cref{append:cmoss_proof}, which is proved by using concentration inequalities, the second term could be bounded as 
\[
    8\sum_{j\in A^*} \E{\zeta_j}\cdot T \leq 32k\sqrt{\delta}\cdot T = 32(\log k)\sqrt{kmT}.
\]
Now we begin to bound the third term $\sum_{t=1}^T\E{\1{ A_t\in \+A_B} \Delta_{A_t}}$, where 
\[
    \+A_B = \set{A\in \+A \mid \Delta_{A} > \max\set{8\sum_{j\in A^*}\zeta_j, (\log k)\sqrt{\frac{ekm}{T}}} }.
\]
This part consists of challenges 2 and 3. For challenge 2, we first show that 
when \Cref{alg:cmoss} chooses $A_t$, for every arm ($\arm_i$ e.g.) in
$A_t\setminus A^*$, we have 
\begin{align*}
    \hat\mu_{i,t} + \rho_{i,t} 
    &\ge \frac{1}{\abs{A_t\setminus A^*}} \tp{\sum_{j\in A^*\setminus A_t} \mu_j - \sum_{j\in A^*} \zeta_j} \\
    &\ge \frac{1}{\abs{A_t\setminus A^*}} \tp{\sum_{i\in A_t\setminus A^*} \mu_i + \frac{7\Delta_{A_t}}{8}}, 
\end{align*}
where the second inequality comes from $\Delta_{A_t} > 8\sum_{j\in A^*}\zeta_j$. For every suboptimal action $A$, define $k_A = \abs{A\setminus A^*}$. So $1\leq k_A \leq k$. We denote $\tp{\mu^A_{i}}_{i\in [k_A]}$ as the mean rewards of the arms in $A\setminus A^*$ where $\mu^A_{k_A} \ge \mu^A_{k_A-1} \ge \cdots \ge \mu^A_{1}$ and $\overline{\mu^A} = \frac{\sum_{i\in [k_A]}\mu^A_i}{k_A}$. Define 
\[
D_{A} = \set{i\in [k_A] \mid \mu^A_i \leq \overline{\mu^A} + \frac{3\Delta_{A}}{4k_A}}
\]  
as the set of arms in $A\setminus A^*$ whose mean reward is no more than the $\overline{\mu^A} + \frac{3\Delta_{A}}{4k_A}$. Then we have the following lemma:
\begin{lemma} \label{lem:actionset_divide}
For every suboptimal action $A \in \+A$, if $\abs{D_{A}} \leq \lfloor\frac{k_A}{2} \rfloor$, then there exists $1\leq \ell_A \leq \lfloor \log k_A \rfloor$ such that 
$$
    \overline {\mu^A} - \mu^A_{\lfloor \frac{k_A}{2^{\ell_A}} \rfloor} \ge \frac{\tp{\sqrt{2}}^{\ell_A}}{16}\cdot \frac{\Delta_{A}}{k_A}.
$$
\end{lemma}
\Cref{lem:actionset_divide} reveals that every suboptimal action $A_t$ could be quantitatively characterized by a fundamental trade-off: it must contain either a few arms with mean rewards far from the action's average mean, $\overline{\mu^A}$, or conversely, many arms with rewards clustered near it. Therefore before the algorithm starts, we could divide the action set $\+A$ into $1+ \lfloor \log k \rfloor$ exhaustive and mutually exclusive sets $\tp{\+A_\ell}_{0\leq \ell \leq \lfloor \log k \rfloor}$, where 
$\+A_0 = \set{A\in \+A \mid \abs{D_A} > \lfloor \frac{k_A}{2} \rfloor}$, $ \+A_{\ell} = \set{A\in \+A \mid \abs{D_A} \leq \lfloor \frac{k_A}{2} \rfloor \land \ell_A = \ell}$. If the choice of $\ell_A$ for $A$ is not unique, we choose $\ell_A$ to be the smallest valid value without loss of generality. Leveraging this classification of actions, we can bound the indicator for selecting a suboptimal action $A_t$:

\begin{lemma}[Indicator decomposition lemma] \label{lem:indicator_decompose}
    For every $0\leq \ell \leq \lfloor \log k \rfloor$, we have
    \begin{align*} 
        &\phantomeq \1{A_t\in \+A_\ell \cap \+A_B } \notag \\
        &\leq \frac{2^{\ell+1}}{k_{A_t}} \sum_{i\notin A^*} \1{  \tp{\hat\mu_{i,t} + \rho_{i,t} \ge \mu_{i} + \frac{\tp{\sqrt{2}}^\ell}{16}\cdot \frac{\Delta_{A_t}}{k_{A_t}}} \land i\in A_t}.
    \end{align*}
\end{lemma}
Note that the coefficient $\frac{2^{\ell+1}}{k_{A_t}}$ is critical to control the quantity of the summation. Then we need to solve challenge 3. We choose to give the actions in every $\+A_\ell$ a regret upper bound. That is, for every $0\leq \ell \leq \lfloor \log k \rfloor$,
$$
    \sum_{t=1}^T \E{\1{A_t\in \+A_\ell \cap \+A_B} \cdot \Delta_{A_t}}
    = O\tp{\sqrt{kmT}}.
$$
For any fixed $0\leq \ell \leq \lfloor \log k\rfloor$, we bound the regret caused by executing actions in $\+A_\ell$. After applying \Cref{lem:indicator_decompose}, we recount the regret by switching the summation of base arms with the summation of actions that a certain arm is involved. To avoid redundant computations for regret, we could arrange the sequence of actions that maximizes the value.

Let each $\arm_i$ be contained in $N_{i,\ell}$ suboptimal actions $A_{i,\ell,1},\cdots, A_{i,\ell,N_{i,\ell}}$ such that $A_{i,\ell,n} \in \+A_\ell, \forall n\in [N_{i,\ell}]$. Denote $\Delta_{A_{i,\ell,n}}$ as $\Delta_{i,\ell,n}$ and $k_{A_{i,\ell,n}}$ as $k_{i,\ell,n}$. Without loss of generality, suppose that $\frac{\Delta_{i,\ell,N_{i,\ell}}}{k_{i,\ell,N_{i,\ell}}} \ge \frac{\Delta_{i,\ell,N_{i,\ell} - 1}}{k_{i,\ell,N_{i,\ell} - 1}} \ge \cdots \ge \frac{\Delta_{i,\ell,1}}{k_{i,\ell,1}}$. Then we have the following lemma:
\begin{lemma}\label{lem:action_arrange}
    For every $0\leq \ell \leq \lfloor \log k \rfloor$, we have
    \begin{align*}
        &\phantomeq \sum_{t=1}^T\E{\1{A_t\in \+A_\ell \cap \+A_B}\Delta_{A_t}} \\
        &\leq 1540 \sum_{i\notin A^*} \Bigg(\sum_{n=2}^{N_{i,\ell}} 
        \frac{k_{i,\ell,n}^2 \ln \tp{\frac{k^2}{\delta} \cdot \frac{\Delta_{i,\ell,n}^2}{k_{i,\ell,n}^2}} }{\Delta_{i,\ell,n}^2}\cdot  \\ 
        &\phantomeq \tp{\frac{\Delta_{i,\ell,n}}{k_{i,\ell,n}} - \frac{\Delta_{i,\ell,n-1}}{k_{i,\ell,n-1}}} + \frac{k_{i,\ell,1} \ln \tp{\frac{k^2}{\delta} \cdot \frac{\Delta_{i,\ell,1}^2}{k_{i,\ell,1}^2}} }{\Delta_{i,\ell,1}} \Bigg).
    \end{align*}
\end{lemma}

By perceiving $\frac{\Delta_{i,\ell,n}}{k_{i,\ell,n}}$ as the variable $x$, we could bound the first term by the integral of $f(x) = \frac{\ln\tp{\frac{k^2 x^2}{\delta}}}{x^2}$ over the relevant domain, while the second term by the maximum of $g(x) = \frac{\ln\tp{\frac{k^2 x^2}{\delta}}}{x}$ over its effective domain. Therefore we have
\begin{align*}
    &\phantomeq \sum_{t=1}^T\E{\1{A_t \in \+A_B } \Delta_{A_t}} \\
    &= \sum_{\ell=0}^{\lfloor \log k \rfloor} \sum_{t=1}^T \E{\1{A_t\in \+A_\ell \cap \+A_B } \Delta_{A_t}} \\
    &= O\tp{(\log k)\sqrt{kmT}}.
\end{align*}

Bring this back to \Cref{line:decompose}, we could obtain $\E{R(T)} = \E{\sum_{t=1}^T \Delta_{A_t}} =  O\tp{(\log k)\sqrt{kmT}}$.

\subsection{Extension to cascading feedback} \label{sec:cmoss_cascading}
In this section we introduce the cascading feedback \cite{KSWA15,KWAS15cascading,LWZC16} and give our regret upper bound. In this model, when $A_t$ is chosen, playing this action means that the player reveals the outcomes of arms one by one following the sequence order until certain stopping condition is satisfied. The feedback is the outcomes of revealed arms and the reward is a function form of all arms in $A_t$. In particular, in the disjunctive form the player stops when the first 1 is revealed, or he reaches the end and gets all feedback. In the conjunctive form, the player stops when the first 0 is revealed or she reaches the end with all 1 outcomes.

Let $p_{A,t}$ be the the probability of observing all base arms of $A$ at round $t$. And denote $p^* = \min_{t\in T}\min_{A\in \+A} p_{A,t}$, which is the minimal probability that an action could have all base arms observed over all time. Then we have the following regret bound: 
 
\begin{theorem} \label{thm:cmoss_ub_cascading}
    For any stochastic combinatorial cascading feedback instance with $m$ arms, time horizon $T$, and the maximum number of arms that could be chosen in a round is $k$, then the expected regret of \Cref{alg:cmoss} is $O\tp{\frac{\log k}{p^*}\sqrt{kmT}}$ when $k \leq \frac{m}{2}$ and $O\left( \frac{m-k}{p^*}\sqrt{\log k\log(m-k) T} \right)$ when $k > \frac{m}{2}$.
\end{theorem}
The proof is given in \Cref{append:cmoss_cascading}. In \Cref{append:cascading}, we also present comparative experiments evaluating the performance of CUCB and CMOSS under cascading feedback.

\section{Experiments}\label{sec:experiments}
\begin{figure*}[t]
    \centering    
    \includegraphics[width=\textwidth]{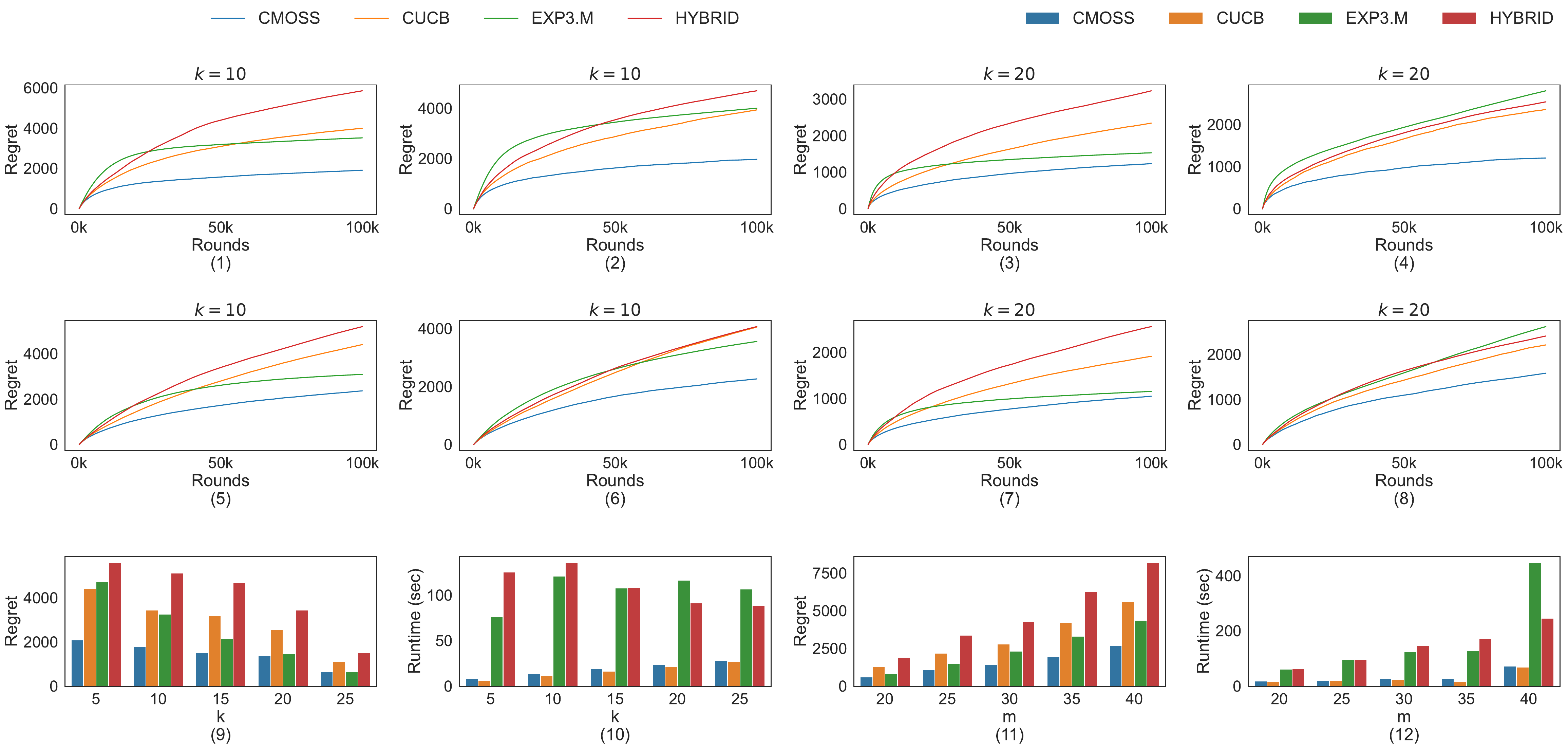}
   \caption{Comparison of CMOSS (blue) with three baselines algorithms under semi-bandit feedback.  
Subplots (1)-(4), (9)-(12) use synthetic dataset; (5)-(8) use Yelp dataset. Subplots (1)-(8) show cumulative regret with \(k=10, m=30\) and \(k=20, m=30\); (9)(10) show ablation studies varying \(k\) (fixed $m=30$), while (11)(12) varying \(m\) (fixed $k=15$).  Initial means of base arms fall within the range \([0, 0.1]\), except for (2)(4)(6)(8), which use the range \([0.3,0.4]\). }
    \label{fig:datasets}
\end{figure*}

In this section, we evaluate CMOSS on both synthetic and real-world datasets under semi-bandit feedback, benchmarking against three representative baselines: CUCB, EXP3.M, and HYBRID. In addition, we conduct ablation studies to investigate the influence of key parameters and assess the robustness of CMOSS under various configurations. All results are averaged over 10 independent runs of $100000$ rounds, including those in \Cref{append:experiments}. Further experimental details and additional ablation results are provided in \Cref{append:experiments}. Visual comparisons across all experiments are summarized in \Cref{fig:datasets}.

For all experiments, we fix the confidence parameter of CMOSS to \(\delta = 0.00001\), which is of the same order of magnitude as the theoretical results $\frac{m(\log k)^2}{kT}, k\leq \frac{m}{2}$ and $\frac{(m-k)^2\log k\log (m-k)}{k^2T}, k>\frac{m}{2}$ (see \Cref{sec:cmoss_algo} and \Cref{append:cmoss_proof}). For EXP3.M, we use a mixing coefficient \(\gamma = 0.01\), as it yields favorable empirical performance in our settings (see \Cref{append:exp3m_gamma}). All rewards follow Bernoulli distributions: pulling arm $\arm_i$ yields reward $1$ with probability \(\mu_i\), and $0$ otherwise, where \(\mu_i\) denotes the true mean (interpretable as the latent click-through probability in recommendation tasks).

\subsection{Comparison with baselines}\label{sec:comparison_with_baselines}
We evaluate CMOSS across both synthetic and real-world bandit environments to assess its effectiveness under diverse reward structures. (1)-(8) of \Cref{fig:datasets} report the cumulative regret over time with cardinality \(k=10\) and \(k=20\) and fixed number of base arms \(m=30\).

In the synthetic setting, the initial click probability \(\mu_i\) for each arm \(\arm_i\) is independently drawn from a uniform distribution. We consider two regimes: a low-probability regime with \(\mu_i \sim \mathtt{U}(0, 0.1)\), and a high-probability regime with \(\mu_i \sim \mathtt{U}(0.3,0.4)\). To further assess the practical effectiveness of our method, we evaluate it on the public Yelp dataset\footnote{\url{http://www.yelp.com/dataset_challenge}}, a classical CMAB real-world dataset. Following \cite{LWK19}, we map user–item affinities to the target regimes \([0, 0.1]\) and \([0.3,0.4]\) via linear rescaling. Full details are provided in \Cref{append_yelp}. The resulting values serve as the initial click probabilities of base arms in the two settings. For each regime, we independently sample $30$ arms.

The low-probability regime \([0, 0.1]\) reflects typical conditions observed in real-world recommendation scenarios, where user interactions such as clicks occur relatively infrequently. This setting provides a realistic environment for evaluating algorithm performance. In contrast, the high-probability regime \([0.3,0.4]\) represents a more idealized scenario with higher click probabilities, allowing us to assess algorithm behavior in more optimistic conditions. By examining both regimes, we gain complementary insights into the robustness and effectiveness of the algorithms across a range of practical and theoretical contexts.

Across all experimental settings in \Cref{tab:formal}, CMOSS consistently achieves the lowest cumulative regret among all baselines under semi-bandit feedback. Under regime $[0,0.1]$, CMOSS reduces regret—relative to CUCB—by \textbf{52.1\%} ($k=10$) and \textbf{47.3\%} ($k=20$) on the synthetic dataset, and by \textbf{46.3\%} ($k=10$) and \textbf{45.2\%} ($k=20$) on the Yelp dataset. 
Under regime $[0.3,0.4]$, CMOSS maintains similarly strong improvements: \textbf{49.9\%} ($k=10$) and \textbf{48.9\%} ($k=20$) on the synthetic dataset, and \textbf{44.1\%} ($k=10$) and \textbf{28.4\%} ($k=20$) on the Yelp dataset. In terms of computational efficiency, CMOSS is slightly slower than CUCB across all configurations, with runtime ratios ranging from \textbf{1.12$\times$} to \textbf{1.15$\times$} at $k=10$ and from \textbf{1.10$\times$} to \textbf{1.17$\times$} at $k=20$. Despite this mild overhead, CMOSS remains \textbf{4--9$\times$} faster than EXP3.M and \textbf{5--10$\times$} faster than HYBRID across every setting, highlighting substantial efficiency advantages over exploration-heavy or hybrid baselines. Overall, these results demonstrate that CMOSS offers a strong balance between regret minimization and runtime efficiency, achieving the best performance across both low- and high-reward regimes and for both $k=10$ and $k=20$ configurations.

\subsection{Ablation studies}\label{sec:ablation}
To further evaluate the robustness and scalability of CMOSS, we conduct ablation studies by varying cardinality \(k\) while fixing base arm size at \(m=30\) and varying base arm size $m$ while fixing cardinality $k=15$ on synthetic dataset. As shown (3)(4) (varying $k$) and (7)(8) (varying $m$) of \Cref{fig:datasets}, we report both cumulative regret and total runtime. Detailed metrics are summarized in \Cref{tab:ablation_k_m}. We also do the same ablation studies on Yelp dataset (see details in \Cref{append_ablation}).

Across both synthetic and Yelp datasets, CMOSS consistently achieves the lowest cumulative regret when varying either $k$ or $m$. Relative to CUCB, CMOSS reduces regret by \textbf{40--55\%} in most configurations, while requiring only a \textbf{10--30\%} increase in runtime. When compared with EXP3.M and HYBRID, CMOSS offers significant improvements: regret is typically \textbf{25--70\%} lower, and runtime is reduced by more than \textbf{65--85\%} across nearly all settings. The ablation studies on the Yelp dataset reveal the same performance patterns as those on the synthetic dataset, confirming that CMOSS remains robust and efficient across different domains and parameter regimes.

In addition to the main experiments, we include supplementary studies in \Cref{append:experiments} to strengthen the completeness and rigor of our evaluation. We first perform a parameter selection study for EXP3.M (see \Cref{append:exp3m_gamma}) to validate the choice of its hyperparameter $\gamma=0.01$ in our setup. We then compare CMOSS and CUCB under a cascading feedback model (see \Cref{append:cascading}) to assess their performance under an alternative user interaction mechanism. In particular, CMOSS achieves a notable reduction in regret under cascading feedback, while incurring only a modest increase in runtime. These additional experiments support both the soundness of our experimental design and the robustness of CMOSS across varied settings.

These results indicate that CMOSS consistently outperforms benchmark algorithms in both regret and runtime efficiency across different datasets and varying values of $k$ and $m$, demonstrating robust effectiveness. This consistent performance highlights CMOSS’s suitability for practical recommendation and ranking tasks, demonstrating its effectiveness and reliability across varying selection sizes.

\section{Conclusion and Future Work}
In this paper, we study the stochastic CMAB problem and propose the CMOSS algorithm, which is not only computationally efficient but also has a regret of $O\big( (\log k)\sqrt{kmT}\big )$ when $k\leq \frac{m}{2}$ and $O\big((m-k)\sqrt{\log k\log(m-k)T}\big )$ when $k>\frac{m}{2}$ under semi-bandit feedback. We then extend our analysis to cascading feedback. Comparative experiments with CUCB, EXP3.M and HYBRID under semi-bandit feedback show that CMOSS has at least one-way significant advantages in cumulative regret and computation cost over the other three algorithms, with the other one showing small differences. 

Several promising directions for future research emerge from this work. First, tightening the $\frac{1}{p^*}$ factor under cascading feedback and other scenarios which contain probabilistically triggered arms, could yield a more general framework and facilitate the analysis of a broader range of problems including episodic reinforcement learning \cite{LWZ24}. Furthermore, recent refinements to the UCB algorithm \cite{L15, MNSR18}, were shown to empirically outperform MOSS in terms of regret. Adapting these more sophisticated UCB variants to the CMAB setting presents a natural and compelling extension.



\bibliography{arxiv}


\newpage
\appendix   
\onecolumn

\section{Techinical Proofs of \Cref{sec:cmoss}} 

\subsection{Proofs of \Cref{sec:cmoss_proof}} \label{append:cmoss_proof}

For completeness and consistency of the paper, we first provide proofs of some basic lemmas (\Cref{lem:general_Hoef}, \Cref{lem:moss-Delta}, \Cref{lem:moss-deviate}) that will be used later.

\begin{lemma}[Hoeffding's lemma, \cite{MU17}]\label{lem:Hoef}
     Let $X$ be a random variable with mean $0$ and $a\leq X\leq b$ holds almost surely. Then for any $\lambda\in \bb R$, 
     \[
        \E{e^{\lambda X}}\leq e^{\frac{\lambda^2(b-a)^2}{8}}.
     \]
\end{lemma}

\begin{lemma}[Generalization of Hoeffding’s inequality, \cite{MU17}]\label{lem:general_Hoef}
    Let $(Y_t)_{t=1}^n$ be a sequence of independent variables in $[0,1]$ with means $(q_t)_{t=1}^n$. let $X_t=Y_t-q_t$ and $S_t=\sum_{s=1}^t X_s$. Then for any $\eps>0$,
    \[
        \Pr{\max_{t\in[n]}S_t\geq \eps}\leq e^{-\frac{2\eps^2}{n}}.
    \]
\end{lemma}

\begin{proof}
First we show that $(e^{\lambda S_t})^n_{t=1}$ is a submartingale for any $\lambda \in \bb{R}$. Let $\+F_t=\sigma\tp{X_1,\dots,X_t}$. Note that $e^{\lambda x}$ is a convex function for all $\lambda \in \bb{R}$. Applying conditional Jensen's inequality, for all $t \ge 1$:
\[
    \E{e^{\lambda S_t}\vert \+F_{t-1}} \ge e^{\lambda \E{S_t\vert \+F_{t-1}}} = e^{\lambda S_{t-1}}.
\]
Therefore $(e^{\lambda S_t})^n_{t=1}$ is a submartingale with regard to $\set{\+F_t}_{t\in[N]}$. From the Doob's submartingale inequality, for any $\lambda > 0$, 
\[
    \Pr{\underset{t\in [n]}{\max}S_t\ge\epsilon} 
    =\Pr{\underset{t\in [n]}{\max}e^{\lambda S_t}\ge e^{\lambda\epsilon}} 
    \leq \frac{\E{e^{\lambda S_n}}}{e^{\lambda\epsilon}}.
\]
According to \Cref{lem:Hoef}, $\E{e^{\lambda X_t}}\leq e^{\frac{\lambda^2(b-a)^2}{8}} = e^{\frac{\lambda^2}{8}}$, then we have
\[
    \Pr{\underset{t\in [n]}{\max}S_t\ge\epsilon}
    \leq \frac{\E{e^{\lambda S_n}}}{e^{\lambda\epsilon}}
    = \frac{\prod_{i=1}^{n}\E{e^{\lambda X_i}} }{e^{\lambda\epsilon}}
    \leq e^{\frac{n\lambda^2}{8}-\epsilon\lambda}.
\]
Set $\lambda = \frac{4\eps}{n}$ to minimize $\frac{n\lambda^2}{8}-\epsilon\lambda$, we obtain 
\[
    \Pr{\underset{t\in [n]}{\max}S_t\ge\epsilon}
    \leq e^{-\frac{2\eps^2}{n}}.
\]
\end{proof}

\begin{lemma}[\cite{LS20}] \label{lem:moss-Delta}
    Define $\zeta_i = \max\tp{0, \mu_i - \min\limits_{t\le T} \tp{\hat{\mu}_{i,t} + \rho_{i,t}}}$. For any $\eps>0$,
    \[
        \Pr{\zeta_i\ge \eps}\le \frac{4\delta}{\eps^2}, \quad 
        \E{\zeta_i} \leq 4\sqrt{\delta}.
    \]
\end{lemma}

\begin{proof}
$\Pr{\zeta_i \ge \eps} = 0$ when $\epsilon \ge \mu_i$. Therefore we only need to prove the case when $0<\epsilon<\mu_i$.
First we have:
\[
    \Pr{\zeta_i\ge\epsilon} \leq \Pr{\max_{s\le T}\tp{\mu_i-\tp{\hat{\mu}_{i,s}+\sqrt{\frac{1}{s}\ln^ + (\frac{1}{\delta s}})}} \ge \eps}.
\]
Without loss of generality, let $T = 2^L$, $S_t = \sum_{s=1}^t \tp{\mu_i-X_i^{(s)}}$. It's easy to verify that $(e^{\lambda S_t})^T_{t=1}$ is a submartingale for any $\lambda \in \bb{R}$. We have:
\begin{align}
    &\phantomeq \Pr{\max_{s\le T}\tp{\mu_i-\tp{\hat{\mu}_{i,s}+\sqrt{\frac{1}{s}\ln^ + (\frac{1}{\delta s}})}} \ge \eps} \nonumber\\
    &\le \sum_{j=1}^L\Pr{\max_{2^{j-1}\le s\le 2^j}s\tp{\mu_i-\hat{\mu}_{i,s}} \ge 2^{j-1}\eps+\sqrt{2^{j-1}\ln^+(\frac{1}{\delta\cdot2^{j}}})} \nonumber \\
    &\le \sum_{j=1}^L\Pr{\max_{1\le s\le 2^j}S_s  \ge 2^{j-1}\eps+\sqrt{2^{j-1}\ln^+(\frac{1}{\delta\cdot2^{j}}})} \nonumber \\
    &\le \sum_{j=1}^L\exp\tp{-\frac{\tp{2^{j-1}\eps+\sqrt{2^{j-1}\ln^+(\frac{1}{\delta\cdot2^{j}}})}^2}{2^{j-1}}} \label{line:hoef_inequality} \\
    &\le \delta \sum_{j=1}^L 2^j e^{-2^{j-1}\eps^2} \notag \\
    &\le \frac{\delta}{e\eps^2} + \delta \int_{0}^{+\infty}2^{x+1} e^{-2^x \eps^2}dx \label{line:unimodal} \\
    &= \frac{\delta}{e\eps^2} + \frac{2\delta}{\ln 2\cdot \eps^2 e^\eps} \leq \frac{4\delta}{\eps^2}, \notag
\end{align}
where \Cref{line:hoef_inequality} follows from \Cref{lem:general_Hoef}. And \Cref{line:unimodal} follows from
\[
    \sum_{x=a}^{b}f(x) \le \max_{x \in [a, b]}f(x) + \int_a^b f(x)dx
\]
when $f(x)$ is unimodal.
Therefore, we have $\Pr{\zeta_i \ge \eps}\le \frac{4\delta}{\eps^2}$. And
\begin{align*}
    \E{ \zeta_i} 
    = \int_0^1\Pr{\zeta_i \ge t}dt \le \int_0^1\min\set{1,\frac{4\delta}{t^2}} dt \leq 4\sqrt{\delta}.
\end{align*}

\end{proof}

\begin{lemma}[\cite{LS20}] \label{lem:moss-deviate}
    Let $X_1,\dots,X_n$  be a sequence of independent random variables in $[0,1]$. Let $\hat\mu_t = \frac{1}{t}\sum_{s=1}^t X_s$. For a fixed $a>1$, $\eps\in (0,1)$, define 
    \[
        \kappa = \sum_{t=1}^T \1{\hat\mu_t + \sqrt{\frac{a}{t}} - \mu \ge \eps},
    \]
    then for any integer $T >1$, we have $\E{\kappa} \leq \frac{3a}{\eps^2}$.
\end{lemma}

\begin{proof}
Denote $S_t = \sum_{s=1}^t(X_s-\mu)$, so $S_t$ is a martingale. We have:
\begin{align*}
    \E{\kappa} &= \sum_{t=1}^n \Pr{S_t \ge t\eps - \sqrt{at}} \\
    &= \sum_{t=1}^{\lfloor a/\eps^2 \rfloor}\tp{1-\Pr{S_t < t\eps - \sqrt{at}}} + \sum_{t=\lfloor a/\eps^2 \rfloor+1}^n \Pr{S_t \ge t\eps - \sqrt{at}}\\
    &\le \frac{a}{\eps^2} + \sum_{t=\lfloor a/\eps^2 \rfloor+1}^n \exp\tp{-\frac{2\tp{t\eps - \sqrt{at}}^2}{t}} \\
    &= \frac{a}{\eps^2} + \sum_{t=\lfloor a/\eps^2 \rfloor+1}^n e^{-2\tp{\sqrt{t}\eps - \sqrt{a}}^2} \\
    &\le  \frac{a}{\eps^2} + \int_{\frac{a}{\eps^2}}^{+\infty}e^{-2\tp{\sqrt{t}\eps - \sqrt{a}}^2}dt  \\
    &=  \frac{a}{\eps^2} + \frac{\sqrt{2\pi}}{2}\cdot\frac{\sqrt{a}}{\eps^2} + \frac{1}{2\eps^2} \leq \frac{3a}{\eps^2}.
\end{align*}
\end{proof}

We begin to give a thorough proof of \Cref{thm:cmoss_ub}. When \Cref{alg:cmoss} chooses $A_t$, for the arms in $A_t\setminus A^*$ and arms in $A^* \setminus A_t$, we have 
\begin{align*}
    \min\limits_{i\in A_t\setminus A^*} \tp{\hat\mu_{i,t} + \rho_{i,t}} 
    \ge \max\limits_{j\in A^*\setminus A_t} \tp{\hat\mu_{j,t} + \rho_{j,t}} \ge \frac{1}{\abs{A^*\setminus A_t}} \sum_{j\in  A^*\setminus A_t} \tp{\hat\mu_{j,t} + \rho_{j,t}}. 
\end{align*}
Then for every $i\in A_t\setminus A^*$, we have
\[
    \hat\mu_{i,t} + \rho_{i,t} 
    \ge \frac{1}{\abs{A^*\setminus A_t}} \min\limits_{s\leq T} \tp{\sum_{j\in A^*\setminus A_t} \tp{\hat\mu_{j,s} + \rho_{j,s}}}
    \ge \frac{1}{\abs{A^*\setminus A_t}} \sum_{j\in A^*\setminus A_t} \min\limits_{s\leq T} \tp{\hat{\mu}_{j,s}+ \rho_{j,s}}.
\]
According to the definition, $\zeta_j \ge \mu_j - \min\limits_{s\le T} \tp{\hat{\mu}_{j,s}+ \rho_{i,s}}, \forall j\in A^*$. Therefore
\begin{align*}
    \hat\mu_{i,t} + \rho_{i,t} 
    &\ge \frac{1}{\abs{A^*\setminus A_t}} \sum_{j\in A^*\setminus A_t} \tp{\mu_j - \zeta_j}  \ge \frac{1}{\abs{A^*\setminus A_t}} \tp{\sum_{j\in A^*\setminus A_t} \mu_j - \sum_{j\in A^*} \zeta_j} 
\end{align*}
Since $\Delta_{A_t} = \sum_{j\in A^*\setminus A_t} \mu_j - \sum_{i\in A_t\setminus A^*} \mu_i$ and $\Delta_{A_t} > 8\sum_{j\in A^*} \zeta_j$, then
\[
\hat\mu_{i,t} + \rho_{i,t} \ge \frac{1}{\abs{ A^*\setminus A_t}} \tp{\sum_{j\in A^*\setminus A_t} \mu_j - \sum_{j\in A^*} \zeta_j} \ge \frac{1}{\abs{ A^*\setminus A_t}} \tp{\sum_{i\in A_t\setminus A^*} \mu_i + \frac{7\Delta_{A_t}}{8}}. 
\]

Remember that we denote $\tp{\mu^A_{i}}_{i\in [k_A]}$ as the mean rewards of the arms in $A\setminus A^*$ where $\mu^A_{k_A} \ge \mu^A_{k_A-1} \ge \cdots \ge \mu^A_{1}$ and $\overline{\mu^A} = \frac{\sum_{i\in [k_A]}\mu^A_i}{k_A}$. Define $D_{A} = \set{i\in [k_A] \mid \mu^A_i \leq \overline{\mu^A} + \frac{3\Delta_{A}}{4k_A}}$ as the set of arms in $A\setminus A^*$ whose mean reward is no more than $\overline{\mu^A} + \frac{3\Delta_{A}}{4k_A}$. Then we have \Cref{lem:actionset_divide}:

\begin{lemma}[\Cref{lem:actionset_divide} restated]
    For every suboptimal action $A \in \+A$, if $\abs{D_{A}} \leq \lfloor\frac{k_A}{2} \rfloor$, then there exists $1\leq \ell_A \leq \lfloor \log k_A \rfloor$ such that 
    \[    
        \overline {\mu^A} - \mu^A_{\lfloor \frac{k_A}{2^{\ell_A}} \rfloor} \ge \frac{\tp{\sqrt{2}}^{\ell_A}}{16}\cdot \frac{\Delta_{A}}{k_A}.
    \]
\end{lemma}

\begin{proof}

Since $\abs{D_{A}} \leq \lfloor\frac{k_A}{2} \rfloor$, we have $\mu^A_{\lfloor\frac{k_A}{2} \rfloor + 1} > \overline {\mu^A} + \frac{3\Delta_A}{4k_A}$. From the definition of $\overline {\mu^A}$, we obtain
\begin{align} 
    \sum_{i=1}^{\lfloor\frac{k_A}{2} \rfloor} \tp{\overline {\mu^A} - \mu^A_i} 
    &= \sum_{i = \lfloor\frac{k_A}{2} \rfloor + 1}^{k_A} \tp{\mu^A_i - \overline {\mu^A}} \notag \\
    &\ge \sum_{i = \lfloor\frac{k_A}{2} \rfloor + 1}^{k_A} \tp{\mu^A_{\lfloor\frac{k_A}{2} \rfloor + 1} - \overline {\mu^A}} \notag \\
    &> \tp{k_A - \lfloor\frac{k_A}{2} \rfloor}\cdot \frac{3\Delta_A}{4k_A} \ge \frac{3\Delta_A}{8}. \label{line:big_sum}
\end{align}

We prove the lemma by contradiction, Assume that for all $1\leq \ell \leq \lfloor \log k_A \rfloor$, $\overline {\mu^A} - \mu^A_{\lfloor \frac{k_A}{2^\ell} \rfloor} < \frac{\tp{\sqrt{2}}^\ell}{16}\cdot \frac{\Delta_{A}}{k_A}$. So for all $\lfloor\frac{k_A}{2^{\ell+1}} \rfloor + 1 \leq i \leq \lfloor\frac{k_A}{2^{\ell}} \rfloor$, we have 
\[
\overline {\mu^A} - \mu^A_i \leq \overline {\mu^A} - \mu^A_{\lfloor\frac{k_A}{2^{\ell+1}} \rfloor} < \frac{\tp{\sqrt{2}}^{\ell+1}}{16}\cdot \frac{\Delta_{A}}{k_A}.
\]

Then we bound $\sum_{i=1}^{\lfloor\frac{k_A}{2} \rfloor} \tp{\overline {\mu^A} - \mu^A_i}$ as follows:
\begin{align*}
    \sum_{i=1}^{\lfloor\frac{k_A}{2} \rfloor} \tp{\overline {\mu^A} - \mu^A_i}
    &= \sum_{j=1}^{\lfloor \log k_A \rfloor} \sum_{i = \lfloor\frac{k_A}{2^{j+1}} \rfloor + 1}^{\lfloor\frac{k_A}{2^j} \rfloor} \tp{\overline {\mu^A} - \mu^A_i} \\
    &< \sum_{j=1}^{\lfloor \log k_A \rfloor} \tp{\lfloor\frac{k_A}{2^j} \rfloor - \lfloor\frac{k_A}{2^{j+1}} \rfloor} \cdot \frac{\tp{\sqrt{2}}^{j+1}}{16}\cdot \frac{\Delta_{A}}{k_A} \\
    &\leq \sum_{j=1}^{\lfloor \log k_A \rfloor}  \tp{\frac{k_A}{2^{j+1}} + \frac{1}{2}}\cdot  \frac{\tp{\sqrt{2}}^{j+1}}{16}\cdot \frac{\Delta_{A}}{k_A} \\
    &= \sum_{j=1}^{\lfloor \log k_A \rfloor} \frac{\Delta_A}{16\tp{\sqrt{2}}^{j+1}} + \sum_{j=1}^{\lfloor \log k_A \rfloor} \frac{\tp{\sqrt{2}}^{j+1} \Delta_A}{32k_A} \\
    &< \frac{\Delta_A}{8} + \frac{\tp{\sqrt{2} + 1} \Delta_A}{16 \sqrt{k_A}} < \frac{3\Delta_A}{8},
\end{align*}
which contradicts \Cref{line:big_sum}. 
\end{proof}

Now we divide the proof into two cases where $k\leq \frac{m}{2}$ and $k > \frac{m}{2}$.

\paragraph{The case where $k \leq \frac{m}{2}$}
We focus on proving the third term $\sum_{t=1}^T\E{\1{ A_t\in \+A_B} \Delta_{A_t}}$, where 
$$
\+A_B = \set{A\in \+A \mid \Delta_{A} > \max\set{8\sum_{j\in A^*}\zeta_j, (\log k)\sqrt{\frac{ekm}{T}}} }.
$$
From \Cref{lem:actionset_divide}, before the algorithm starts, we could divide the action set $\+A$ into $1+ \lfloor \log k \rfloor$ exhaustive and mutually exclusive sets $\tp{\+A_\ell}_{0\leq \ell \leq \lfloor \log k \rfloor}$ since $1\leq k_A\leq k$, where 
$$
\+A_0 = \set{A\in \+A \mid \abs{D_A} > \lfloor \frac{k_A}{2} \rfloor}, \+A_{\ell} = \set{A\in \+A \mid \abs{D_A} \leq \lfloor \frac{k_A}{2} \rfloor \land \ell_A = \ell}.
$$
Leveraging this classification of actions, we can bound the indicator for selecting a suboptimal action $A_t$ through this lemma:

\begin{lemma}[\Cref{lem:indicator_decompose} restated]
    For every $0\leq \ell \leq \lfloor \log k \rfloor$, we have
    \begin{align*} 
        &\phantomeq \1{A_t\in \+A_\ell \cap \+A_B} \leq \frac{2^{\ell+1}}{k_{A_t}} \sum_{i\notin A^*} \1{\tp{ \hat\mu_{i,t} + \rho_{i,t} \ge \mu_{i} + \frac{\tp{\sqrt{2}}^\ell}{16}\cdot \frac{\Delta_{A_t}}{k_{A_t}}} \land i\in A_t}.
    \end{align*}
\end{lemma}

\begin{proof}
By the definition of $D_A$, we know that for every $A_t\in \+A_0$, there are at least $\lfloor \frac{k_{A_t}}{2} \rfloor + 1$ arms ($\arm_i$ e.g.) in $A_t\setminus A^*$ satisfies
$$
\hat\mu_{i,t} + \rho_{i,t} \ge \overline {\mu^{A_t}} + \frac{7\Delta_{A_t}}{8k_{A_t}} \ge \mu_i + \frac{\Delta_{A_t}}{8k_{A_t}} \ge \mu_i + \frac{\Delta_{A_t}}{16k_{A_t}}.
$$
Therefore we have
$$
    \1{A_t\in \+A_0 \cap \+A_B} \leq \frac{2}{k_{A_t}} \sum_{i\notin A^*} \1{\tp{\hat\mu_{i,t} + \rho_{i,t} \ge \mu_i + \frac{\Delta_{A_t}}{16k_{A_t}}} \land i\in A_t}.
$$
For every $A_t \in \+A_\ell, 1\leq \ell \leq \lfloor \log k \rfloor$, there are at least $\lfloor \frac{k_{A_t}}{2^{\ell}} \rfloor$ arms ($\arm_i$ e.g.) in $A_t \setminus A^*$ satisfies
$$
\hat\mu_{i,t} + \rho_{i,t} \ge \overline {\mu^{A_t}} + \frac{7\Delta_{A_t}}{8k_{A_t}} \ge \mu_i + \frac{\tp{\sqrt{2}}^\ell}{16}\cdot \frac{\Delta_{A_t}}{k_{A_t}}.
$$
Similarly we have
\begin{align*}
    \1{A_t\in \+A_\ell \cap \+A_B}
    &\leq \frac{1}{\lfloor \frac{k_{A_t}}{2^\ell} \rfloor} \sum_{i\notin A^*} \1{\tp{\hat\mu_{i,t} + \rho_{i,t} \ge \mu_{i} + \frac{\tp{\sqrt{2}}^\ell}{16}\cdot \frac{\Delta_{A_t}}{k_{A_t}}} \land i\in A_t} \\
    &\leq \frac{2^{\ell+1}}{k_{A_t}} \sum_{i\notin A^*} \1{\tp{\hat\mu_{i,t} + \rho_{i,t} \ge \mu_{i} + \frac{\tp{\sqrt{2}}^\ell}{16}\cdot \frac{\Delta_{A_t}}{k_{A_t}}} \land i\in A_t}.
\end{align*}
In summary, for every $0\leq \ell \leq \lfloor \log k \rfloor$, 
\begin{align*} 
    \1{A_t\in \+A_\ell \cap \+A_B} \leq \frac{2^{\ell+1}}{k_{A_t}} \sum_{i\notin A^*} \1{\tp{\hat\mu_{i,t} + \rho_{i,t} \ge \mu_{i} + \frac{\tp{\sqrt{2}}^\ell}{16}\cdot \frac{\Delta_{A_t}}{k_{A_t}}} \land i\in A_t}.
\end{align*}
\end{proof}

\Cref{lem:indicator_decompose} quantitatively decouples the constrains for $A_t$ into every arm in $A_t\setminus A^*$. Next we give the actions in every $\+A_\ell$ a regret upper bound. That is, for every $0\leq \ell \leq \lfloor \log k \rfloor$, we want to prove 
\[
    \sum_{t=1}^T \E{\1{A_t\in \+A_\ell \cap \+A_B}\cdot \Delta_{A_t}} = O\tp{\sqrt{kmT}}.
\]

Recall that each $\arm_i$ is contained in $N_{i,\ell}$ suboptimal actions $A_{i,\ell,1},\cdots, A_{i,\ell,N_{i,\ell}}$ such that $A_{i,\ell,n} \in \+A_\ell, \forall n\in [N_{i,\ell}]$. Denote $\Delta_{A_{i,\ell,n}}$ as $\Delta_{i,\ell,n}$ and $k_{A_{i,\ell,n}}$ as $k_{i,\ell,n}$. Without loss of generality, suppose that $\frac{\Delta_{i,\ell,N_{i,\ell}}}{k_{i,\ell,N_{i,\ell}}} \ge \frac{\Delta_{i,\ell,N_{i,\ell} - 1}}{k_{i,\ell,N_{i,\ell} - 1}} \ge \cdots \ge \frac{\Delta_{i,\ell,1}}{k_{i,\ell,1}}$. Then we have:

\begin{align}
    &\phantomeq \sum_{t=1}^T\E{\1{A_t\in \+A_\ell \cap \+A_B}\cdot \Delta_{A_t}} \notag \\
    &\leq \sum_{t=1}^T\E{\sum_{i\notin A^*} \1{\hat\mu_{i,t} + \rho_{i,t} \ge \mu_i + \frac{(\sqrt{2})^\ell}{16}\cdot \frac{\Delta_{A_t}}{k_{A_t}}}\cdot \frac{2^{\ell+1}\Delta_{A_t}}{k_{A_t}} \mid i\in A_t} \notag \\ 
    &= \sum_{i\notin A^*} \sum_{t=1}^T \sum_{n=1}^{N_{i,\ell}} \E{\1{A_t = A_{i,\ell,n}} \cdot \1{\hat\mu_{i,t} + \rho_{i,t} \ge \mu_i + \frac{(\sqrt{2})^\ell}{16}\cdot \frac{\Delta_{i,\ell,n}}{k_{i,\ell,n}} }\cdot \frac{2^{\ell+1}\Delta_{i,\ell,n}}{k_{i,\ell,n}} \mid i\in A_t} \notag \\ 
    &= 2^{\ell+1}\sum_{i\notin A^*} \E{\sum_{t=1}^T \sum_{n=1}^{N_{i,\ell}} \1{A_t = A_{i,\ell,n}} \cdot \1{\hat\mu_{i,t} + \rho_{i,t} \ge \mu_i + \frac{(\sqrt{2})^\ell}{16}\cdot \frac{\Delta_{i,\ell,n}}{k_{i,\ell,n}}} \cdot \frac{\Delta_{i,\ell,n}}{k_{i,\ell,n}} \mid i\in A_t}, \label{line:maximize_main}
\end{align}
where the first inequality is from \Cref{lem:indicator_decompose}. To avoid redundant computations for regret at \Cref{line:maximize_main}, we could arrange the sequence of actions that maximizes the value. Moreover, applying \Cref{lem:moss-deviate}, we have the following lemma:

\begin{lemma}[\Cref{lem:action_arrange} restated]
    For every $0\leq \ell \leq \lfloor \log k \rfloor$, we have
    \begin{align*}
        &\phantomeq \sum_{t=1}^T\E{\1{A_t\in \+A_\ell \cap \+A_B} \Delta_{A_t}} \\
        &\leq 1540 \sum_{i\notin A^*} \tp{\sum_{n=2}^{N_{i,\ell}} 
        \frac{k_{i,\ell,n}^2 \ln \tp{\frac{k^2}{\delta} \cdot \frac{\Delta_{i,\ell,n}^2}{k_{i,\ell,n}^2}} }{\Delta_{i,\ell,n}^2} \tp{\frac{\Delta_{i,\ell,n}}{k_{i,\ell,n}} - \frac{\Delta_{i,\ell,n-1}}{k_{i,\ell,n-1}}} + \frac{k_{i,\ell,1} \ln \tp{\frac{k^2}{\delta} \cdot \frac{\Delta_{i,\ell,1}^2}{k_{i,\ell,1}^2}} }{\Delta_{i,\ell,1}}  }.
    \end{align*}
\end{lemma}

\begin{proof}
Since $\frac{\Delta_{i,\ell,N_{i,\ell}}}{k_{i,\ell,N_{i,\ell}}} \ge \frac{\Delta_{i,\ell,N_{i,\ell} - 1}}{k_{i,\ell,N_{i,\ell} - 1}} \ge \cdots \ge \frac{\Delta_{i,\ell,1}}{k_{i,\ell,1}}$, when $\hat\mu_{i,t} + \rho_{i,t} - \mu_i \ge \frac{(\sqrt{2})^\ell}{16}\cdot \frac{\Delta_{i,\ell,N_{i,\ell}}}{k_{i,\ell,N_{i,\ell}}}$, we should choose $A_t = A_{i,\ell,N_{i,\ell}}$. And when $\frac{(\sqrt{2})^\ell}{16}\cdot \frac{\Delta_{i,\ell,N_{i,\ell-1}}}{k_{i,\ell,N_{i,\ell-1}}} \leq \hat\mu_{i,t} + \rho_{i,t} - \mu_i  < \frac{(\sqrt{2})^\ell}{16}\cdot \frac{\Delta_{i,\ell,N_{i,\ell}}}{k_{i,\ell,N_{i,\ell}}}$, we should choose $A_t = A_{i,\ell,N_{i,\ell}-1}$, and so on. By choosing $A_1,\cdots A_T$ to maximize the value of 
\[
    \sum_{t=1}^T \sum_{n=1}^{N_{i,\ell}} \1{A_t = A_{i,\ell,n}} \cdot \1{\hat\mu_{i,t} + \rho_{i,t} \ge \mu_i + \frac{(\sqrt{2})^\ell}{16}\cdot \frac{\Delta_{i,\ell,n}}{k_{i,\ell,n}}} \cdot \frac{\Delta_{i,\ell,n}}{k_{i,\ell,n}},
\]
we obtain that 
\begin{align}
    &\phantomeq \sum_{t=1}^T \sum_{n=1}^{N_{i,\ell}} \1{A_t = A_{i,\ell,n}} \cdot \1{\hat\mu_{i,t} + \rho_{i,t} \ge \mu_i + \frac{(\sqrt{2})^\ell}{16}\cdot \frac{\Delta_{i,\ell,n}}{k_{i,\ell,n}}} \cdot \frac{\Delta_{i,\ell,n}}{k_{i,\ell,n}} \notag \\
    &\leq \sum_{t=1}^T \bigg( \1{\hat\mu_{i,t} + \rho_{i,t} \ge \mu_i + \frac{(\sqrt{2})^\ell}{16}\cdot \frac{\Delta_{i,\ell,N_{i,\ell}}}{k_{i,\ell,N_{i,\ell}}}} \cdot \frac{\Delta_{i,\ell,N_{i,\ell}}}{k_{i,\ell,N_{i,\ell}}}   \notag \\
    &\phantomeq + \sum_{n=1}^{N_{i,\ell}-1} \1{\frac{(\sqrt{2})^\ell}{16}\cdot \frac{\Delta_{i,\ell,n}}{k_{i,\ell,n}} \leq \hat\mu_{i,t} + \rho_{i,t} - \mu_i < \frac{(\sqrt{2})^\ell}{16}\cdot \frac{\Delta_{i,\ell,n+1}}{k_{i,\ell,n+1}}} \frac{\Delta_{i,\ell,n}}{k_{i,\ell,n}} \bigg)  \notag \\
    &= \sum_{t=1}^T \bigg( \sum_{n=2}^{N_{i,\ell}} \1{\hat\mu_{i,t} + \rho_{i,t} \ge \mu_i + \frac{(\sqrt{2})^\ell}{16}\cdot \frac{\Delta_{i,\ell,n}}{k_{i,\ell,n}}} \tp{\frac{\Delta_{i,\ell,n}}{k_{i,\ell,n}} - \frac{\Delta_{i,\ell,n-1}}{k_{i,\ell,n-1}}} \notag \\
    &\phantomeq + \1{\hat\mu_{i,t} + \rho_{i,t} \ge \mu_i + \frac{(\sqrt{2})^\ell}{16}\cdot \frac{\Delta_{i,\ell,1}}{k_{i,\ell,1}}} \cdot \frac{\Delta_{i,\ell,1}}{k_{i,\ell,1}} \bigg), \label{line:indicator}
\end{align}
where \Cref{line:indicator} is due to $\1{x \leq y < z} = \1{y \ge x} - \1{y \ge z}$ for any $y$ if $x<z$. 

From \Cref{lem:moss-deviate}, for every $\Delta_{i,\ell,n}, n\in [N_{i,\ell}]$, we bound 
\[
\E{\sum_{t=1}^T \1{\hat\mu_{i,t} + \rho_{i,t} \ge \mu_i + \frac{(\sqrt{2})^\ell}{16}\cdot \frac{\Delta_{i,\ell,n}}{k_{i,\ell,n}}} \mid i\in A_t}
\] 
as follows:
\begin{align*}
    &\phantomeq \E{\sum_{t=1}^T \1{\hat\mu_{i,t} + \rho_{i,t} \ge \mu_i + \frac{(\sqrt{2})^\ell}{16}\cdot \frac{\Delta_{i,\ell,n}}{k_{i,\ell,n}}} \mid i\in A_t} \\
    &\leq \E{\sum_{T_{i,t}=1}^T \1{\frac{\sum_{s=1}^{T_{i,t}}X_i^{(s)}}{T_{i,t}} + \sqrt{\frac{\ln^+ \tp{\frac{1}{\delta T_{i,t}}}}{T_{i,t}}}  \ge \mu_i + \frac{(\sqrt{2})^\ell}{16}\cdot \frac{\Delta_{i,\ell,n}}{k_{i,\ell,n}} }} \\
    &\leq \frac{1}{\Delta_{i,\ell,n}^2} + \E{\sum_{T_{i,t} > \frac{1}{\Delta_{i,\ell,n}^2}}^T \1{ \frac{\sum_{s=1}^{T_{i,t}}X_i^{(s)}}{{T_{i,t}}} + \sqrt{\frac{\ln \tp{\frac{\Delta_{i,\ell,n}^2}{\delta}} }{T_{i,t}}} \ge \mu_i + \frac{(\sqrt{2})^\ell}{16}\cdot \frac{\Delta_{i,\ell,n}}{k_{i,\ell,n}} } } \\
    &\leq \frac{1}{\Delta_{i,\ell,n}^2} + \frac{768 k_{i,\ell,n}^2 \ln \tp{\frac{\Delta_{i,\ell,n}^2}{\delta}} }{2^\ell \Delta_{i,\ell,n}^2} \leq \frac{770 k_{i,\ell,n}^2 \ln \tp{\frac{\Delta_{i,\ell,n}^2}{\delta}} }{2^\ell \Delta_{i,\ell,n}^2}.
\end{align*}
Combining this with \Cref{line:maximize_main} and \Cref{line:indicator}, we have
\begin{align*}
    &\phantomeq \sum_{t=1}^T\E{\1{A_t\in \+A_\ell \cap \+A_B}\cdot \Delta_{A_t}} \\
    &\leq 2^{\ell+1}\sum_{i\notin A^*} \*E \Bigg[ \sum_{t=1}^T \sum_{n=1}^{N_{i,\ell}} \1{\hat\mu_{i,t} + \rho_{i,t} \ge \mu_i + \frac{(\sqrt{2})^\ell}{16}\cdot \frac{\Delta_{i,\ell,n}}{k_{i,\ell,n}}} \cdot \frac{\Delta_{i,\ell,n}}{k_{i,\ell,n}} \Bigg \vert \, i\in A_t \Bigg] \\
    &\leq 2^{\ell+1}\sum_{i\notin A^*} \*E \Bigg[ \sum_{t=1}^T \bigg( \sum_{n=2}^{N_{i,\ell}} \1{\hat\mu_{i,t} + \rho_{i,t} \ge \mu_i + \frac{(\sqrt{2})^\ell}{16}\cdot \frac{\Delta_{i,\ell,n}}{k_{i,\ell,n}}} \tp{\frac{\Delta_{i,\ell,n}}{k_{i,\ell,n}} - \frac{\Delta_{i,\ell,n-1}}{k_{i,\ell,n-1}}} \notag \\
    &\phantomeq \phantomeq  + \1{\hat\mu_{i,t} + \rho_{i,t} \ge \mu_i + \frac{(\sqrt{2})^\ell}{16}\cdot \frac{\Delta_{i,\ell,1}}{k_{i,\ell,1}}} \cdot \frac{\Delta_{i,\ell,1}}{k_{i,\ell,1}} \bigg) \Bigg \vert \, i\in A_t \Bigg] \\
    &= 2^{\ell+1}\sum_{i\notin A^*} \bigg( \sum_{n=2}^{N_{i,\ell}} \*E \Bigg[ \sum_{t=1}^T  \1{\hat\mu_{i,t} + \rho_{i,t} \ge \mu_i + \frac{(\sqrt{2})^\ell}{16}\cdot \frac{\Delta_{i,\ell,n}}{k_{i,\ell,n}}} 
    \Bigg \vert \, i\in A_t \Bigg]
    \tp{\frac{\Delta_{i,\ell,n}}{k_{i,\ell,n}} - \frac{\Delta_{i,\ell,n-1}}{k_{i,\ell,n-1}}} \notag \\
    &\phantomeq \phantomeq  + \*E \Bigg[\sum_{t=1}^T \1{\hat\mu_{i,t} + \rho_{i,t} \ge \mu_i + \frac{(\sqrt{2})^\ell}{16}\cdot \frac{\Delta_{i,\ell,1}}{k_{i,\ell,1}}} \Bigg \vert \, i\in A_t \Bigg] \cdot \frac{\Delta_{i,\ell,1}}{k_{i,\ell,1}} \bigg) \\
    &\leq 770\cdot 2^{\ell+1} \sum_{i\notin A^*} \tp{\sum_{n=2}^{N_{i,\ell}} 
    \frac{k_{i,\ell,n}^2 \ln \tp{\frac{\Delta_{i,\ell,n}^2}{\delta}} }{2^\ell \Delta_{i,\ell,n}^2} \tp{\frac{\Delta_{i,\ell,n}}{k_{i,\ell,n}} - \frac{\Delta_{i,\ell,n-1}}{k_{i,\ell,n-1}}} + \frac{k_{i,\ell,1}^2 \ln \tp{\frac{\Delta_{i,\ell,1}^2}{\delta}} }{2^\ell \Delta_{i,\ell,1}^2} \cdot \frac{\Delta_{i,\ell,1}}{k_{i,\ell,1}} } \\ 
    &\leq 1540 \sum_{i\notin A^*} \tp{\sum_{n=2}^{N_{i,\ell}} 
    \frac{k_{i,\ell,n}^2 \ln \tp{\frac{k^2}{\delta} \cdot \frac{\Delta_{i,\ell,n}^2}{k_{i,\ell,n}^2}} }{\Delta_{i,\ell,n}^2} \tp{\frac{\Delta_{i,\ell,n}}{k_{i,\ell,n}} - \frac{\Delta_{i,\ell,n-1}}{k_{i,\ell,n-1}}} + \frac{k_{i,\ell,1} \ln \tp{\frac{k^2}{\delta} \cdot \frac{\Delta_{i,\ell,1}^2}{k_{i,\ell,1}^2}} }{\Delta_{i,\ell,1}}  }.
\end{align*}
\end{proof}

Perceiving $\frac{\Delta_{i,\ell,n}}{k_{i,\ell,n}}$ as the variable $x$, we could bound the first term by the integral of $f(x) = \frac{\ln\tp{\frac{k^2 x^2}{\delta}}}{x^2}$ over $x > \frac{\Delta_{i,\ell,1}}{k_{i,\ell,1}} > (\log k)\sqrt{\frac{em}{kT}}$, while the second term is bounded by the maximum of $g(x) = \frac{\ln\tp{\frac{k^2 x^2}{\delta}}}{x}$ over $x > (\log k)\sqrt{\frac{em}{kT}}$. $f(x)$ has been plotted in \Cref{fig:function}. It is easy to verify that $f(x)$ is monotonically decreasing and positive when $x > \sqrt{\frac{e\delta}{k^2} } = (\log k)\sqrt{\frac{em}{k^3T}}$. Additionally we have $\int f(x) dx = \frac{\ln \frac{\delta}{e^2k^2} - 2\ln x}{x}$. Thus the summation could be bounded by the integral of $f(x)$ on $((\log k)\sqrt{\frac{em}{kT}},\infty)$. Also $g(x)$ is monotonically decreasing when $x > (\log k)\sqrt{\frac{em}{kT}}$, which gives an upper bound for the second term $\frac{k_{i,\ell,1} \ln \tp{\frac{k^2}{\delta} \cdot \frac{\Delta_{i,\ell,1}^2}{k_{i,\ell,1}^2}} }{\Delta_{i,\ell,1}}$. Thus we have
\begin{align*}
    \sum_{t=1}^T\E{\1{A_t\in \+A_\ell}\Delta_{A_t}}  
    &\leq 1540 \sum_{i\notin A^*} \tp{\int_{(\log k)\sqrt{\frac{em}{kT}}}^\infty f(x) dx + \frac{\ln ek^2}{(\log k)\sqrt{\frac{em}{kT}}} } \\ 
    &= 1540 \sum_{i\notin A^*} \tp{ \frac{\ln \frac{\delta}{e^2k^2} - 2\ln x}{x} \Big|_{(\log k)\sqrt{\frac{e m}{kT}}}^\infty + \frac{\ln ek^2}{\log k} \sqrt{\frac{kT}{e m}} } \\ 
    &= 1540\sum_{i\notin A^*} \tp{\frac{\ln e^4k^4}{\log k}\sqrt{\frac{k T}{em}}} = O\tp{\sqrt{kmT}}.
\end{align*}

\begin{figure}[H]
    \centering
    \begin{tikzpicture}[scale=0.8]
        \begin{axis}[
                axis lines = middle,
                xlabel = \(x\),
                ylabel = \(f(x)\),
                xmin = 0.25, xmax = 2,   
                ymin = -1, ymax = 4, 
                domain = 0.25:2,         
                samples = 200,           
                smooth,
                grid = both,             
                minor x tick num = 4,    
                minor y tick num = 4,   
                xticklabels = \empty,    
                yticklabels = \empty,   
                extra y ticks = {0},     
                width = 12cm,
                height = 8cm,
                restrict y to domain=-1000:1000, 
                xlabel style = {at={(ticklabel* cs:1)}, anchor=north west},
                ylabel style = {at={(ticklabel* cs:1)}, anchor=south west}
            ]

            \def\theta{0.1}

            \addplot[thick,black]{ln(x^2/\theta)/x^2};
            \addlegendentry{\( \frac{\ln \left( \frac{k^2x^2}{\delta} \right)}{x^2} \)}

            \pgfmathsetmacro\maxX{sqrt(e*\theta)} 
            \pgfmathsetmacro\maxY{ln(\maxX^2/\theta)/\maxX^2} 
            \draw[thick, dashed, color=black] (0.52,3.68) -- (0.52,0);
            \node at (0.52, -0.3) {$\sqrt{\frac{e\delta}{k^2}}$};
            \fill[black] (0.52,0) circle (2pt);

            \draw[thick, dashed, color=black] (0.72,3.18) -- (0.72,0);
            \node at (0.72, -0.3) {$\sqrt{e\delta}$};
            \fill[black] (0.72,0) circle (2pt);

            \pgfmathsetmacro\fA{ln(1.9^2/\theta)/1.9^2} 
            \pgfmathsetmacro\fB{ln(1.72^2/\theta)/1.72^2} 
            \pgfmathsetmacro\fC{ln(1.52^2/\theta)/1.52^2} 
            \pgfmathsetmacro\fD{ln(1.32^2/\theta)/1.32^2} 

            \pgfmathsetmacro\fE{ln(1.2^2/\theta)/1.2^2} 
            \pgfmathsetmacro\fF{ln(1.11^2/\theta)/1.11^2} 
    
            \pgfmathsetmacro\fG{ln(1^2/\theta)/1^2} 
            \pgfmathsetmacro\fH{ln(0.85^2/\theta)/0.85^2} 

            \node at (1.96, \fA/2) {$\bm{\cdots}$};
            
            \draw[thick, dashed, color=black] (1.9,\fA) -- (1.9,0);
            \node at (1.9, -0.3) {$\frac{\Delta_{i,\ell,n}}{k_{i,\ell,n}}$};
            \fill[black] (1.9,0) circle (2pt);

            \draw[thick, dashed, color=black] (1.72,\fB) -- (1.72,0);
            \node at (1.72, -0.3) {$\frac{\Delta_{i,\ell,n-1}}{k_{i,\ell,n-1}}$};
            \fill[black] (1.72,0) circle (2pt);

            \draw[pattern=north east lines, pattern color=black] (1.72,0) rectangle (1.9,\fA);

            \draw[thick, dashed, color=black] (1.52,\fC) -- (1.52,0);
            \node at (1.52, -0.3) {$\frac{\Delta_{i,\ell,n-2}}{k_{i,\ell,n-2}}$};
            \fill[black] (1.52,0) circle (2pt);

            \draw[pattern=north east lines, pattern color=black] (1.52,0) rectangle (1.72,\fB);

            \draw[thick, dashed, color=black] (1.32,\fD) -- (1.32,0);
            \node at (1.32, -0.3) {$\frac{\Delta_{i,\ell,n-3}}{k_{i,\ell,n-3}}$};
            \fill[black] (1.32,0) circle (2pt);

            \draw[pattern=north east lines, pattern color=black] (1.32,0) rectangle (1.52,\fC);

            \draw[thick, dashed, color=black] (1.2,\fE) -- (1.2,0);
            \draw[thick, dashed, color=black] (1.11,\fF) -- (1.11,0);
            \draw[pattern=north east lines, pattern color=black] (1,0) rectangle (1.11,\fF);
            \draw[pattern=north east lines, pattern color=black] (1.2,0) rectangle (1.32,\fD);
            \node at (1.16, \fF/2) {$\bm{\cdots}$};

            \draw[thick, dashed, color=black] (1,\fG) -- (1,0);
            \node at (1, -0.3) {$\frac{\Delta_{i,\ell,2}}{k_{i,\ell,2}}$};
            \fill[black] (1,0) circle (2pt);
            
            \draw[thick, dashed, color=black] (0.85,\fH) -- (0.85,0);
            \node at (0.85, -0.3) {$\frac{\Delta_{i,\ell,1}}{k_{i,\ell,1}}$};
            \fill[black] (0.85,0) circle (2pt);

            \draw[pattern=north east lines, pattern color=black] (0.85,0) rectangle (1,\fG);

        \end{axis}
    \end{tikzpicture}
\caption{The plot of $f(x) = \frac{\ln\tp{\frac{k^2 x^2}{\delta}}}{x^2}$.}
\label{fig:function}
\end{figure}
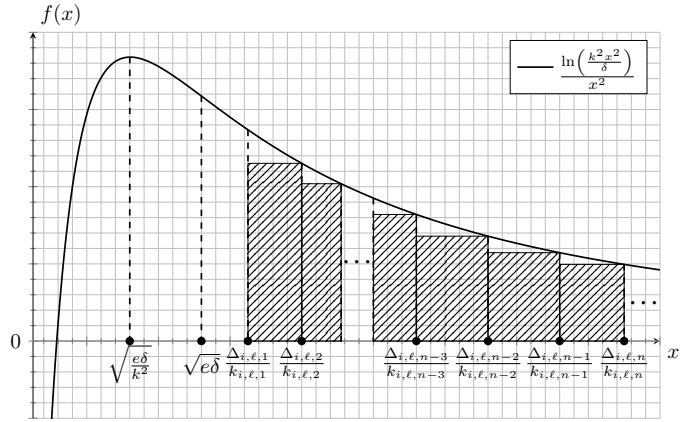

Therefore we have
\begin{align*}
    \sum_{t=1}^T\E{\1{A_t \in \+A_B} \Delta_{A_t}} 
    &= \sum_{\ell=0}^{\lfloor \log k \rfloor} \sum_{t=1}^T \E{\1{A_t\in \+A_\ell \cap \+A_B} \Delta_{A_t}} \\
    &= O\tp{(\log k)\sqrt{kmT}}.
\end{align*}

Bring this back to \Cref{line:decompose}, we could obtain $\E{R(T)} = \E{\sum_{t=1}^T \Delta_{A_t}} =  O\tp{(\log k)\sqrt{kmT}}$.

\paragraph{The case where $k > \frac{m}{2}$}
In this case, we need to prove that the regret of CMOSS is $O\tp{(m-k) \sqrt{\log k\log (m-k) T}}$. Here $\delta$ is optimized to be $\frac{(m-k)^2\log k \log(m-k)}{k^2 T}$. Now we let $\+A_B = \set{A\in \+A \mid \Delta_A > \max\set{8\sum_{j=1}^k \zeta_j,  (m-k)\sqrt{\frac{e\log k \log(m-k)}{T}} }}$. And we decompose the overall regret as follows:
\begin{align}
    \sum_{t=1}^T\E{\Delta_{A_t}} 
    &= \sum_{t=1}^T\E{\1{\Delta_{A_t}\leq (m-k)\sqrt{\frac{e\log k \log(m-k)}{T}} } \Delta_{A_t}} + \sum_{t=1}^T\E{\1{\Delta_{A_t}\leq 8\sum_{j=1}^k \zeta_j}\Delta_{A_t}} \notag \\
    &+ \sum_{t=1}^T\E{\1{\Delta_{A_t} > \max\set{8\sum_{j=1}^k \zeta_j,  (m-k)\sqrt{\frac{e\log k \log(m-k)}{T}} }} \Delta_{A_t}} \notag \\
    &\leq (m-k)\sqrt{e\log k \log(m-k)T} + 8\sum_{j=1}^k  \E{\zeta_j}\cdot T + \sum_{t=1}^T\E{\1{A_t \in \+A_B} \Delta_{A_t}}. \label{line:decompose_2}
\end{align}
From \Cref{lem:moss-Delta}, the second term could be bounded:
$8\sum_{j=1}^k  \E{\zeta_j}\cdot T \leq 8k\sqrt{\delta}\cdot T = 8(m-k) \sqrt{\log k\log (m-k) T}$. 

Since now the range of $k_A$ changes to $1\leq k_A \leq m-k$, from \Cref{lem:actionset_divide}, we could divide the action set $\+A$ into $1+ \lfloor \log (m-k) \rfloor$ exhaustive and mutually exclusive sets $\tp{\+A_\ell}_{0\leq \ell \leq \lfloor \log (m-k) \rfloor}$, where 
$$
\+A_0 = \set{A\in \+A \mid \abs{D_A} > \lfloor \frac{k_A}{2} \rfloor}, \+A_{\ell} = \set{A\in \+A \mid \abs{D_A} \leq \lfloor \frac{k_A}{2} \rfloor \land \ell_A = \ell}.
$$
Furthermore, we could similarly prove \Cref{lem:indicator_decompose} in this case: for every $0\leq \ell \leq \lfloor \log (m-k) \rfloor$, 
\begin{align} 
    \1{A_t\in \+A_\ell \cap \+A_B } \leq \frac{2^{\ell+1}}{k_{A_t}} \sum_{i\notin A^*} \1{\hat\mu_{i,t} + \rho_{i,t} \ge \mu_{i} + \frac{\tp{\sqrt{2}}^\ell}{16}\cdot \frac{\Delta_{A_t}}{k_{A_t}}\land i\in A_t}.
\end{align}
Next we give the actions in every $\+A_\ell$ a regret upper bound. That is, for every $0\leq \ell \leq \lfloor \log (m-k) \rfloor$, we need to prove 
\[
    \sum_{t=1}^T \E{\1{A_t\in \+A_\ell \cap \+A_B}\cdot \Delta_{A_t}} = O\tp{(m-k)\sqrt{\frac{T\log k}{\log (m-k)}}}.
\]
Following the same process of the proof of \Cref{lem:action_arrange}, by slightly changing the upper bound of $k_A$ from $k$ to $m-k$, we have that, for every $0\leq \ell \leq \lfloor \log (m-k) \rfloor$, 
\begin{align*}
    &\phantomeq \sum_{t=1}^T\E{\1{A_t\in \+A_\ell}\Delta_{A_t}} \\
    &\leq 1540 \sum_{i\notin A^*} \tp{\sum_{n=2}^{N_{i,\ell}} 
    \frac{k_{i,\ell,n}^2 \ln \tp{\frac{(m-k)^2}{\delta} \cdot \frac{\Delta_{i,\ell,n}^2}{k_{i,\ell,n}^2}} }{\Delta_{i,\ell,n}^2} \tp{\frac{\Delta_{i,\ell,n}}{k_{i,\ell,n}} - \frac{\Delta_{i,\ell,n-1}}{k_{i,\ell,n-1}}} + \frac{k_{i,\ell,1} \ln \tp{\frac{(m-k)^2}{\delta} \cdot \frac{\Delta_{i,\ell,1}^2}{k_{i,\ell,1}^2}} }{\Delta_{i,\ell,1}}  }.
\end{align*}

Next we consider the properties of the function $f(x) = \frac{\ln\tp{\frac{(m-k)^2 x^2}{\delta}}}{x^2}, x > \frac{\Delta_{i,\ell,1}}{k_{i,\ell,1}} > \sqrt{\frac{e\log k \log(m-k)}{T}} = \frac{k\sqrt{e\delta}}{m-k}$. It is easy to verify that $f(x)$ is monotonically decreasing and positive when $x > \sqrt{\frac{e\delta}{(m-k)^2} } = \frac{\sqrt{e\delta}}{m-k}$. Additionally we have $\int f(x) dx = \frac{\ln \frac{\delta}{e^2 (m-k)^2} - 2\ln x}{x}$. Also $g(x) = \frac{\ln\tp{\frac{(m-k)^2 x^2}{\delta}}}{x}$ is monotonically decreasing when $x > \frac{k\sqrt{e\delta}}{m-k}$. Thus the summation could be bounded by the integral of $f(x)$ on $(\frac{k\sqrt{e\delta}}{m-k},\infty)$
\begin{align*}
    \sum_{t=1}^T\E{\1{A_t\in \+A_\ell}\Delta_{A_t}}  
    &\leq 1540 \sum_{i\notin S^*} \tp{\int_{\frac{k\sqrt{e\delta}}{m-k} }^\infty f(x) dx + \frac{\ln ek^2}{ \frac{k\sqrt{e\delta}}{m-k} }} \\ 
    &\leq 1540 \sum_{i\notin S^*} \tp{ \frac{\ln \frac{\delta}{e^2(m-k)^2} - 2\ln x}{x} \Big|_{ \frac{k\sqrt{e\delta}}{m-k} }^\infty + \ln ek^2 \sqrt{\frac{T}{e\log k \log (m-k)}} } \\ 
    &= 1540 \sum_{i\notin S^*} \tp{\ln e^4k^4 \sqrt{\frac{T}{e\log k \log (m-k)}} } \\ 
    &= O\tp{(m-k)\sqrt{\frac{T\log k}{\log (m-k)}}}.
\end{align*}
Applying this lemma, we have
\begin{align*}
    \sum_{t=1}^T \E{\1{A_t\in  \+A_B}\cdot \Delta_{A_t}} 
    &= \sum_{\ell=0}^{\lfloor \log (m-k) \rfloor}  \sum_{t=1}^T \E{\1{A_t\in \+A_\ell \cap \+A_B}\cdot \Delta_{A_t}} \\
    &= O\tp{(m-k) \sqrt{\log k \log(m-k) T}}.
\end{align*}

Bring this back to \Cref{line:decompose_2}, we could obtain that the regret of \Cref{alg:cmoss} when $k> \frac{m}{2}$ is $O\tp{(m-k) \sqrt{\log k \log(m-k) T}}$.

\subsection{Proofs of \Cref{sec:cmoss_cascading}}
\label{append:cmoss_cascading}
We take the disjunctive cascading bandits for example. In this case $r(A_t, \vmu)=\prod_{i\in A_t} (1-\mu_i)$. Since $\+A$ contains all subsets of $[m]$, $A^*$ consists of the the $k$ arms with the highest expected means. Denote $A^* = \set{j_1,\cdots,j_k}$ where $\mu_{j_1} \leq \mu_{j_2}\leq \cdots \leq \mu_{j_k}$ and $A_t = \set{i_1,\cdots,i_k}$ where $\mu_{i_1} \leq \cdots \leq \mu_{i_k}$. Thus
\begin{align*}
    \abs{r(A^*, \vmu)- r(A_t, \vmu))} 
    &= \abs{\prod_{n=1}^k (1-\mu_{j_n})   - \prod_{n=1}^k (1-\mu_{i_n})} \\
    &\leq \sum_{l=0}^{k-1} \abs{ \prod_{n=1}^l (1-\mu_{i_n}) \prod_{n=l+1}^k (1-\mu_{j_n}) -  \prod_{n=1}^{l+1} (1-\mu_{i_n}) \prod_{n=l+2}^k (1-\mu_{j_n}) } \\
    &=  \sum_{l=0}^{k-1} \tp{ \prod_{n=1}^l (1-\mu_{i_n}) \cdot\prod_{n=l+2}^k (1-\mu_{j_n}) \cdot \abs{\mu_{j_{l+1}} - \mu_{i_{l+1}} }}  \\
    &\leq \sum_{n=1}^k  \abs{\mu_{j_n} - \mu_{i_n}} = \sum_{n=1}^k \mu_{j_n} - \sum_{n=1}^k \mu_{i_n}.
\end{align*}

Let $p_{A,t}$ be the the probability of observing all base arms of $A$ at round $t$. And denote $p^* = \min_{t\in T}\min_{A\in \+A} p_{A,t}$. Therefore in each round $\1{\tau_t = A_t} \ge p^*$. Let $\+F_t$ be the history before the learning agent chooses action at time $t$. Thus $\+F_t$ contains feedback information at all time $s < t$. We could bound the regret as follows
\begin{align}
\E{R(T)} 
&= \E{\E{\sum_{t=1}^T \Delta_{A_t} \mid \+F_t}} \notag \\
&= \E{\E{\sum_{t=1}^T \Delta_{A_t} \cdot \E{\frac{1}{p_{A_t,t}}\1{\tau_t = A_t}\mid A_t} \mid \+F_t}} \notag \\
&\leq \frac{1}{p^*} \E{\E{\sum_{t=1}^T \Delta_{A_t} \1{\tau_t = A_t} \mid \+F_t}} \notag \\ 
&=  \begin{cases}
        O\tp{\frac{\log k}{p^*}\sqrt{kmT}} &, k\leq \frac{m}{2} \\
        O\tp{\frac{m-k}{p^*}\sqrt{\log k \log(m-k) T}} &, k > \frac{m}{2} 
    \end{cases} \label{line:qstar}
\end{align}
where \Cref{line:qstar} applies \Cref{thm:cmoss_ub} to this case.

For the conjunctive case, $r(A_t, \vmu) = \prod_{i\in A_t} \mu_i$. The rest analysis follows the same pattern as the disjunctive case.

\section{Details of CUCB, EXP3.M and HYBRID} \label{append:algo}
In this section we present the complete formulations of the CUCB algorithm \cite{CWY13,KWAS15,LWZ24,WC17}, the EXP3.M algorithm \cite{UNK10} and the HYBRID algorithm \cite{ZLW19}.

\subsection{The CUCB algorithm}

\begin{algorithm}[H]
\caption{Combinatorial Upper Confidence Bound (CUCB)} 
\label{alg:cucb}
\textbf{Input}:{$m, k$.}
    \begin{algorithmic}[1]
        \State For each arm $i$, $T_{i}\leftarrow 0$
        \Comment{maintain the total number of times arm $i$ is played so far}
        \State For each arm $i$, $\hat \mu_{i}\leftarrow 1$
        \Comment{maintain the empirical mean of $\arm_i$}
        \For{$t=1, 2, 3, \ldots$}
            \State For each arm $i\in [m]$, $\rho_i \leftarrow \sqrt{\frac{3\ln t}{2T_i}}$
            \Comment{the confidence radius, $\rho_i=+\infty$ if $T_i=0$}
            \State For each arm $i\in [m]$, $\bar{\mu}_{i} \leftarrow  \min\left \{\hat{\mu}_{i} + \rho_i,1\right \}, \bar\vmu \leftarrow (\bar{\mu}_{1}, \cdots, \bar{\mu}_{m})$   
            \State $A_t \leftarrow \argmax_{A\in \+A} r(A, \bar\vmu)$
            \State Pull $A_t$ of base arms and obtain the triggered set $\tau_t$. Observe feedback $X_i^{(t)}$ for all $i\in \tau_t$
            \State For every $i \in \tau_t$, update $T_i$ and $\hat{\mu}_{i}$: 
            $T_i \leftarrow T_i + 1$, $\hat{\mu}_{i} \leftarrow \hat{\mu}_{i} + (X_i^{(t)}-\hat{\mu}_{i})/T_i$
        \EndFor
    \end{algorithmic}
\end{algorithm}

 The regret bound of the CUCB algorithm is $O\tp{\sqrt{kmT\log T}}$ under semi-bandit feedback. A detailed analysis of the regret is provided in \cite{KWAS15}.

\subsection{The EXP3.M algorithm}

\begin{algorithm}[H]
\caption{EXP3.M} \label{alg:exp3m}
\textbf{Input}:{mixing coefficient $\gamma \in (0,1]$}
    \begin{algorithmic}[1]
        \State $\omega_i(1) = 1$ for $i=1,\cdots,m$
        \Comment{Initialization}
        \For{$t=1, 2, 3, \ldots$}
            \If{$\max_{j\in[m]}\omega_j(t)\ge \tp{\frac{1}{k} - \frac{\gamma}{m}}\sum_{i=1}^m \frac{\omega_i(t)}{1-\gamma} $}
                \State Decide $\alpha_t$ so as to satisfy         $\frac{\alpha_t}{\sum_{\omega_i(t)\ge \alpha_t}\alpha_t + \sum_{\omega_i(t)< \alpha_t}\omega_i(t)} = \tp{\frac{1}{k} - \frac{\gamma}{m}}\cdot\frac{1}{1-\gamma}$ \label{line:decide_alpha}
                \State Set $S_0(t) = \set{i\mid \omega_i(t)\ge \alpha_t}$ and $\omega_i'(t) = \alpha_t$ for $i\in S_0(t)$
            \Else
                \State Set $S_0(t) = \emptyset$
            \EndIf
            \State Set  $\omega_i'(t) = \omega_i(t)$ for $i\in [m]\setminus S_0(t)$
            \State Set $p_i(t) = k\tp{\tp{1-\gamma}\frac{\omega_i'(t)}{\sum_{j=1}^m\omega_j'(t)} + \frac{\gamma}{m}}$ for $i\in[m]$ \label{line:p_i}
            \State Set $S(t) = \*{DepRound} \tp{k,(p_1(t),p_2(t),\cdots,p_m(t))}$
            \State Receive rewards $x_i(t)$ for $i\in S(t)$
            \For{$i = 1,\cdots, m$}   \label{line:for_loop}
                \State Set $\hat x_i(t) = \begin{cases}
                    \frac{x_i(t)}{p_i(t)}, & i\in S(t) \\
                    0, & o.w.
                \end{cases}$
                \State Set $\omega_i(t+1) = \begin{cases}
                    \omega_i(t)\exp\tp{\frac{k\gamma \hat x_i(t)}{m}}, & i\notin S_0(t) \\
                    \omega_i(t), & o.w.
                \end{cases}$
            \EndFor
        \EndFor
    \end{algorithmic}
\end{algorithm}

In \Cref{alg:exp3m}, if all $\omega_i(t)$ are less than $\tp{\frac{1}{k} - \frac{\gamma}{m}}\sum_{i=1}^m \frac{\omega_i(t)}{1-\gamma}$, then 
every $p_i(t)$ calculated at Line 9 of \Cref{line:p_i} is less than $1$ and $S_0(t)=\emptyset$. Otherwise it chooses a appropriate threshold $\alpha_t$ that the temporal weight $\omega_i'(t)$ is set to $\alpha_t$ for $i\in S_0(t)$ and $\omega_i(t)$ for $i\notin S_0(t)$. Thus $p_i(t) =1$ for $i\in S_0(t)$, which means that $S_0(t) \subseteq S(t)$. And $\omega_i(t)$ is not updated for $i\in S_0(t)$ since it is large to some extent.

The algorithm $\*{DepRound}$ is designed as follows:
\begin{algorithm}[H]
\caption{$\*{DepRound}$}  \label{alg:depround}
\textbf{Input}:{ Natural number $k$, $\tp{p_1,\cdots,p_m}$ with $\sum_{i=1}^m p_i=k, 0\leq p_i \leq 1, \forall i\in [m]$}
    \begin{algorithmic}[1]
        \While{there exists $i\in [m]$ with $0<p_i<1$}
            \State Choose distinct $i$ and $j$ with $0 <p_i < 1$ and $0<p_j < 1$ 
            \State Set $\alpha = \min\set{1-p_i,p_j}$, $\beta=\min\set{p_i,1-p_j}$
            \State Update $p_i,p_j$ as
            $(p_i,p_j) = \begin{cases}
                (p_i+\alpha, p_j-\alpha), & w.p.\; \frac{\beta}{\alpha+\beta} \\
                (p_i-\beta, p_j+\beta), & w.p. \; \frac{\alpha}{\alpha+\beta} 
            \end{cases}$
        \EndWhile
    \end{algorithmic}
\Output{$\set{i\mid p_i=1, i\in[m]}$}
\end{algorithm}

In \Cref{alg:depround}, $(p_1,\cdots,p_m)$ is probabilistically updated until all the components are $0$ or $1$ while keeping the condition that $\sum_{i=1}^m p_i=k$. Therefore if there exists $i\in[m]$ with $0 <p_i < 1$, then there must be some $j\neq i$ with $0 <p_j < 1$ otherwise the summation  $\sum_{i=1}^m p_i$ is not an integer. Inside the while-loop, the iteration lasts at most $m$ times because at least one of $p_i$ and $p_j$ becomes $0$ or $1$ each time. The regret bound of the EXP3.M algorithm is $O\tp{\sqrt{kmT\log \tp{\frac{m}{k}}}}$ and the detailed analysis is provided in \cite{UNK10}.

\subsection{The HYBRID algorithm}

\begin{algorithm}[H]
\caption{FTRL with hybrid regularizer for semi-bandits}
\label{alg:hybrid}
\textbf{Input}: { $0<\gamma \leq 1$, sampling scheme $P$}
\begin{algorithmic}[1]    
    \State $\hat{L}_0 = (0, \ldots, 0), \eta_t = 1/\sqrt{t}$
    \For{$t=1, 2, \dots$}
    \State Compute $ x_t = \argmin\limits_{ x \in \conv(\mathcal{A})} \big\langle x,\hat{L}_{t-1}\big\rangle + \eta_t^{-1}\Psi( x)$
    \State Sample $A_t \sim P(x_t)$\;
    \State Observe $o_t = A_t \circ \ell_t$\;
    \State Construct estimator $\hat\ell_t,\; \forall i:\; \hat\ell_{ti} = \frac{(o_{ti}+1)\mathbb{I}_t(i)}{ x_{ti}}-1 $ \;
    \State Update $\hat L_t = \hat L_{t-1}+\hat\ell_t$\;
    \EndFor
\end{algorithmic}
\end{algorithm}

Here $\conv(\+A)$ is the convex hull of $\+A$. And the hybrid regularizer 
\begin{equation}
\Psi( x) =\Regularizer
\end{equation} 
with a parameter $0<\gamma\leq 1$ to be chosen as follows
\[
    \gamma=\begin{cases}
    1 &\mbox{ if } k\leq \frac{m}{2}\\
    \min\{1,1/\sqrt{\log(m/(m-k))}\} &\mbox{ otherwise.} 
    \end{cases}
\] 

The regret bound of the HYBRID algorithm is $O\tp{\sqrt{kmT}}$ when $k \leq \frac{m}{2}$ and $O\left((m-k)\sqrt{\log\tp{\frac{m}{m-k}} T} \right)$ when $k> \frac{m}{2}$. And the detailed analysis is provided in \cite{ZLW19}.

\section{Experiments}
\label{append:experiments}
The experiments are performed on a Dell Inspiration 16 PLUS 7620 laptop equipped with a 12th Generation Intel(R) Core(TM) i7-12700H CPU running at 2.30 GHz, alongside 16 GB of RAM. The operating system utilized is Windows 11, and the Python version employed is 3.11.7.

\subsection{Comparison with baselines and beyond}
\subsubsection{Synthetic dataset}
As described in the main text, in the synthetic setting, we simulate a combinatorial bandit problem with $m = 30$ base arms and cardinality $k = 10$ and $k = 20$. The initial value for each arm is independently drawn from uniform distributions $\mathtt{U}(0, 0.1)$ and  $\mathtt{U}(0.3,0.4)$. Each algorithm is run for $100000$ rounds, and all results are averaged over $10$ independent runs.

As shown in \Cref{tab:formal}, CMOSS attains notable regret reductions relative to the baselines under the ($[0,0.1]$) initialization. For ($k=10$), the reductions are \textbf{52.1\%} (CUCB), \textbf{47.2\%} (EXP3.M), and \textbf{67.3\%} (HYBRID). For ($k=20$), the improvements remain significant at \textbf{47.3\%}, \textbf{19.4\%}, and \textbf{61.7\%}, respectively. This trend persists in ($[0.3,0.4]$) regime: CMOSS outperforms CUCB, EXP3.M, and HYBRID by \textbf{49.9\%} / \textbf{50.8\%} / \textbf{58.0\%} for ($k=10$), and \textbf{48.9\%} / \textbf{56.9} / \textbf{52.5\%} for ($k=20$). In terms of computational efficiency, CMOSS runs only slightly slower than CUCB—by approximately \textbf{2.22s} and \textbf{2.85s} in the ($[0,0.1]$) setting and \textbf{1.88s} and \textbf{1.56s} in ($[0.3,0.4]$). However, it is markedly faster than both EXP3.M and HYBRID, achieving speedups on the order of \textbf{6–9$\times$} relative to EXP3.M and \textbf{4–9$\times$} relative to HYBRID. These results demonstrate that CMOSS delivers superior regret performance while retaining computational efficiency.

\subsubsection{Real-world dataset} \label{append_yelp}
To simulate a real-world bandit environment, we use the public Yelp dataset\footnote{\url{http://www.yelp.com/dataset_challenge}}, a classical dataset in the CMAB literature. Following the setup in \cite{LWK19}, we construct a filtered subset consisting of 1500 users and 10000 items from the original dataset. Each user and item is represented as a $50$-dimensional real-valued feature vector, normalized to unit $\ell_2$ norm. For each user–item pair \((u, a)\), we compute an affinity score as the dot product \(\mathbf{u}^\top \mathbf{a}\), where \(\mathbf{u}, \mathbf{a} \in \mathbb{R}^{50}\).

To match the synthetic regimes used in our experiments, the raw affinity scores are min-max normalized and linearly rescaled into two ranges: a low-probability regime \([0, 0.1]\) and a high-probability regime \([0.3,0.4]\). Specifically, let \(\min\) and \(\max\) denote the minimum and maximum dot products across all user–item pairs. The scores are transformed as follows:
\begin{align}
    \text{score}_{\text{low}}(u,a) &= 0.1 \cdot \frac{\mathbf{u}^\top \mathbf{a} - \min}{\max - \min},\\
    \text{score}_{\text{high}}(u,a) &= 0.3 + 0.1 \cdot \frac{\mathbf{u}^\top \mathbf{a} - \min}{\max - \min}.
\end{align}

The resulting values are interpreted as the initial click probabilities for base arms. For each regime, we independently sample 30 such values as base arms using a fixed random seed (2026) to ensure reproducibility.

As shown in \Cref{tab:formal}, the results on the Yelp dataset exhibit trends consistent with the synthetic setting. Under the $[0,0.1]$ initialization, CMOSS reduces cumulative regret \textbf{46.3\%} (CUCB), \textbf{23.4\%} (EXP3.M), and \textbf{54.5\%} (HYBRID). For ($k=20$), CMOSS maintains strong improvements of \textbf{45.2\%}, \textbf{8.9\%}, and \textbf{58.9\%}, respectively. A similar pattern holds under the $[0.3,0.4]$ regime: CMOSS outperforms CUCB, EXP3.M, and HYBRID by $\textbf{44.1\%}$ / \textbf{36.3\%} / \textbf{44.3\%} for ($k=10$), and \textbf{28.4\%} / \textbf{39.4\%} / \textbf{34.3\%} for ($k=20$). Runtime comparisons show that CMOSS remains close to CUCB-slower by only 2-3 seconds across settings-while being significantly faster than EXP3.M and HYBRID, with speedups between \textbf{6$\times$} and \textbf{12$\times$}. These findings confirm that CMOSS maintains its advantage not only in controlled synthetic environments but also in real-world recommendation scenarios, even under the exaggerated high-probability regime ($[0.3,0.4]$).

{\rmfamily
\begin{table}[H]
\caption{Semi-bandit feedback: cumulative regret and runtime (in seconds) of algorithms under two initial base arm value ranges: \([0,0.1]\) and \([0.3,0.4]\), evaluated on both synthetic and Yelp datasets and fixed $m=30$. 
\\
$\star$ indicates our proposed method (CMOSS).
}
\centering
\renewcommand{\arraystretch}{1.2}
\resizebox{\linewidth}{!}{%
\setlength{\tabcolsep}{1.5mm}
\begin{tabular}{c c c c c c c c c c}
\toprule
\multirow{3}{*}{\textbf{Metric}} & \multirow{3}{*}{\textbf{Algorithm}} & \multicolumn{4}{c}{Range $[0,0.1]$} & \multicolumn{4}{c}{Range $[0.3,0.4]$} \\
\cmidrule(lr){3-6} \cmidrule(lr){7-10}
 & & \multicolumn{2}{c}{Synthetic} & \multicolumn{2}{c}{Yelp} & \multicolumn{2}{c}{Synthetic} & \multicolumn{2}{c}{Yelp} \\ \cmidrule(lr){3-4} \cmidrule(lr){5-6} \cmidrule(lr){7-8} \cmidrule(lr){9-10}
 & & $k=10$ & $k=20$ & $k=10$ & $k=20$ & $k=10$ & $k=20$ & $k=10$ & $k=20$ \\
\midrule
\multirow{4}{*}{Regret} 
 & \textbf{CMOSS}$\star$  & \textbf{1917.013} & \textbf{1237.831} & \textbf{2368.325} & \textbf{1050.454} & \textbf{1974.476} & \textbf{1205.354} & \textbf{2265.129} & \textbf{1586.061} \\
 & CUCB & 4000.132 & 2348.979 & 4407.244 & 1916.620 & 3945.550 & 2361.077 & 4050.910 & 2214.263 \\
 & EXP3.M & 3523.649 & 1535.670 & 3094.234 & 1154.182 & 4009.885 & 2800.222 & 3555.937 & 2621.060 \\
 & HYBRID & 5858.858 & 3233.744 & 5197.661 & 2561.223 & 4706.725 & 2538.204 & 4067.236 & 2411.709 \\
\midrule
\multirow{4}{*}{Runtime (sec)} 
 & \textbf{CMOSS}$\star$ & \textbf{17.372} & \textbf{30.670} & \textbf{13.529} & \textbf{24.213} & \textbf{12.601} & \textbf{21.948} & \textbf{14.609} & \textbf{27.363} \\
 & CUCB & 15.155 & 27.816 & 11.422 & 21.108 & 10.724 & 20.391 & 11.739 & 23.882 \\
 & EXP3.M & 104.291 & 92.774 & 79.239 & 78.895 & 79.592 & 87.363 & 95.136 & 110.756 \\
 & HYBRID & 121.126 & 111.946 & 108.540 & 105.515 & 95.413 & 76.987 & 143.529 & 163.490 \\
\bottomrule
\end{tabular}
}
\label{tab:formal}
\end{table}
}

\subsection{Ablation studies and beyond} \label{append_ablation}

We provide the complete results of our ablation study evaluating the impact of two key parameters: the cardinality \( k \) of the super arm and the total number of base arms \( m \). We consider two controlled experiments with initial base arms in $[0,0.1]$:
\begin{itemize}
    \item \textbf{Varying \( k \)}: \( k \in \{5, 10, 15, 20, 25\} \), with \( m = 30 \) fixed.
    \item \textbf{Varying \( m \)}: \( m \in \{20, 25, 30, 35, 40\} \), with \( k = 15 \) fixed.
\end{itemize}

On the \textbf{synthetic} dataset ($\mathtt{U}(0,0.1)$), CMOSS consistently achieves the lowest cumulative regret across all configurations. When varying $k$, CMOSS outperforms CUCB, EXP3.M, and HYBRID by up to \textbf{52.7\%}, \textbf{55.7\%}, and \textbf{67.4\%}, respectively. When varying $m$, the gains become even more pronounced—CMOSS reduces regret by as much as \textbf{53.3\%} over CUCB, \textbf{40.6\%} over EXP3.M, and \textbf{68.7\%} over HYBRID. In terms of runtime, CMOSS remains close to CUCB while being substantially faster than EXP3.M and HYBRID across all settings. Moreover, CMOSS exhibits stable runtime as $m$ increases, demonstrating strong scalability. On the \textbf{Yelp} dataset (rescaled to $[0,0.1]$), we observe consistent trends. When varying $k$, CMOSS achieves regret reductions of up to \textbf{51.0\%} over CUCB, \textbf{39.4\%} over EXP3.M, and \textbf{70.0\%} over HYBRID. When varying $m$, CMOSS again demonstrates clear advantages—outperforming CUCB by up to \textbf{46.9\%}, \textbf{23.1\%}, and \textbf{61.5\%}, respectively. Runtime comparisons further confirm CMOSS’s efficiency: it remains close to CUCB while being substantially faster than both EXP3.M and HYBRID. Additionally, CMOSS maintains stable runtime as either $k$ or $m$ increases, showcasing its practical scalability on real-world data. Complete numerical results are reported in \Cref{tab:ablation_k_m} and \Cref{tab:ablation_k_m_yelp}.

{\rmfamily
\begin{table}[H]
\caption{Ablation study of semi-bandit feedback: varying $k$ (with $m=30$ fixed) and varying $m$ (with $k=15$ fixed) on synthetic dataset. Results correspond to subplots (3)(4)(7)(8) in \Cref{fig:datasets}.\\  $\star$ indicates our proposed method (CMOSS).}
\centering
\renewcommand{\arraystretch}{1.2}
\setlength{\tabcolsep}{1.5mm}
\resizebox{\textwidth}{!}{
\begin{tabular}{c c c c c c c c c c c c}
\toprule
\multirow{3}{*}{\textbf{Metric}} & \multirow{3}{*}{\textbf{Algorithm}} & \multicolumn{5}{c}{Varying $k$ (Synthetic)} & \multicolumn{5}{c}{Varying $m$ (Synthetic)} \\
\cmidrule(lr){3-7} \cmidrule(lr){8-12}
 & & $k=5$ & $k=10$ & $k=15$ & $k=20$ & $k=25$ & $m=20$ & $m=25$ & $m=30$ & $m=35$ & $m=40$ \\
\midrule
\multirow{4}{*}{Regret} 
 & \textbf{CMOSS}$\star$ 
 & \textbf{2096.027} & \textbf{1783.914} & \textbf{1523.988} & \textbf{1376.970} & \textbf{661.014}
 & \textbf{618.792} & \textbf{1092.666} & \textbf{1440.614} & \textbf{1965.203} & \textbf{2695.970} \\
 & CUCB 
 & 4429.727 & 3446.924 & 3184.634 & 2578.753 & 1131.007
 & 1296.234 & 2182.487 & 2791.642 & 4208.573 & 5586.506 \\
 & EXP3.M 
 & 4736.340 & 3271.894 & 2162.514 & 1462.691 & 659.203
 & 836.489 & 1500.403 & 2327.747 & 3305.964 & 4369.808 \\
 & HYBRID 
 & 5596.137 & 5123.518 & 4676.580 & 3456.369 & 1517.378
 & 1918.722 & 3376.063 & 4285.878 & 6283.876 & 8184.837 \\
\midrule
\multirow{4}{*}{Runtime (sec)} 
 & \textbf{CMOSS}$\star$ 
 & \textbf{8.606} & \textbf{13.616} & \textbf{19.021} & \textbf{23.666} & \textbf{28.344}
 & \textbf{18.795} & \textbf{21.210} & \textbf{29.320} & \textbf{29.423} & \textbf{72.478} \\
 & CUCB 
 & 6.574 & 11.654 & 16.397 & 21.287 & 26.991
 & 16.218 & 21.619 & 25.239 & 18.210 & 69.932 \\
 & EXP3.M 
 & 76.208 & 120.789 & 107.957 & 116.241 & 106.603
 & 61.540 & 96.724 & 124.403 & 129.031 & 447.195 \\
 & HYBRID 
 & 125.178 & 135.696 & 108.276 & 91.548 & 88.490
 & 63.963 & 96.495 & 147.526 & 172.044 & 245.882 \\
\bottomrule
\end{tabular}}
\label{tab:ablation_k_m}
\end{table}
}
{\rmfamily
\begin{table}[H]
\caption{Ablation study of semi-bandit feedback: varying $k$ (with $m=30$ fixed) and varying $m$ (with $k=15$ fixed) on Yelp dataset. \\ $\star$ indicates our proposed method (CMOSS).}
\centering
\renewcommand{\arraystretch}{1.2}
\setlength{\tabcolsep}{1.5mm}
\resizebox{\textwidth}{!}{
\begin{tabular}{c c c c c c c c c c c c}
\toprule
\multirow{3}{*}{\textbf{Metric}} & \multirow{3}{*}{\textbf{Algorithm}} & \multicolumn{5}{c}{Varying $k$ (Yelp)} & \multicolumn{5}{c}{Varying $m$ (Yelp)} \\
\cmidrule(lr){3-7} \cmidrule(lr){8-12}
 & & $k=5$ & $k=10$ & $k=15$ & $k=20$ & $k=25$ & $m=20$ & $m=25$ & $m=30$ & $m=35$ & $m=40$ \\
\midrule
\multirow{4}{*}{Regret} 
 & \textbf{CMOSS}$\star$ 
 & \textbf{2436.508} & \textbf{2309.895} & \textbf{1705.547} & \textbf{1085.540} & \textbf{583.928}
 & \textbf{659.741} & \textbf{1324.735} & \textbf{1779.739} & \textbf{1890.972} & \textbf{2604.883} \\
 & CUCB 
 & 4020.879 & 4208.577 & 3074.333 & 2175.718 & 1191.837
 & 1162.516 & 2450.990 & 3299.031 & 3559.703 & 4558.815 \\
 & EXP3.M 
 & 4021.209 & 2988.275 & 2025.406 & 1216.121 & 663.273
 & 658.334 & 1409.273 & 2088.690 & 2424.766 & 3387.021 \\
 & HYBRID 
 & 4969.380 & 5078.731 & 3960.924 & 2890.771 & 1944.711
 & 1548.021 & 3338.523 & 4347.092 & 4913.297 & 5471.507 \\
\midrule
\multirow{4}{*}{Runtime (sec)} 
 & \textbf{CMOSS}$\star$ 
 & \textbf{15.016} & \textbf{21.507} & \textbf{22.474} & \textbf{28.051} & \textbf{34.512}
 & \textbf{18.526} & \textbf{26.895} & \textbf{25.122} & \textbf{28.148} & \textbf{82.126} \\
 & CUCB 
 & 7.448 & 13.630 & 19.281 & 25.876 & 31.900
 & 16.236 & 24.373 & 23.632 & 28.706 & 71.921 \\
 & EXP3.M 
 & 91.734 & 106.868 & 127.099 & 166.739 & 86.960
 & 69.510 & 95.093 & 132.767 & 109.091 & 435.470 \\
 & HYBRID 
 & 119.877 & 125.436 & 104.920 & 118.024 & 155.868
 & 78.343 & 105.469 & 156.025 & 167.520 & 379.948 \\
\bottomrule
\end{tabular}}
\label{tab:ablation_k_m_yelp}
\end{table}
}

\subsection{Parameter selection for EXP3.M}\label{append:exp3m_gamma}

The performance of EXP3.M in semi-bandits is highly sensitive to the choice of its mixing coefficient \(\gamma\), which controls the balance between exploration and exploitation. To determine an appropriate setting for our environment, we conduct an ablation over $\gamma \in \{0.001, 0.01, 0.1\}$ with cardinality $k \in \{10,20\}$. Each configuration is run for a horizon of $T = 100000$ and averaged over 10 independent repetitions.

As shown in \Cref{tab:gamma} and \Cref{fig:gamma}, $\gamma=0.01$ consistently achieves the lowest cumulative regret across both initialization regimes. Under $\mathtt{U}(0,0.1)$, it reduces regret by $83.9\%$ and $80.3\%$ relative to $\gamma=0.001$ and $\gamma=0.1$ for $k=10$, and by $84.1\%$ and $44.8\%$ for $k=20$. A similar pattern holds under $\mathtt{U}(0.3,0.4)$: $\gamma=0.01$ improves over $\gamma=0.001$ and $\gamma=0.1$ by $75.6\%$ and $81.3\%$ for $k=10$, and by $67.3\%$ and $89.3\%$ for $k=20$. Based on these findings, we adopt \(\gamma = 0.01\) as the default parameter for all experiments involving EXP3.M, as it strikes the most favorable balance between exploration and exploitation in our stochastic environment.

\begin{figure}[H]
    \centering    \includegraphics[width=0.9\textwidth]{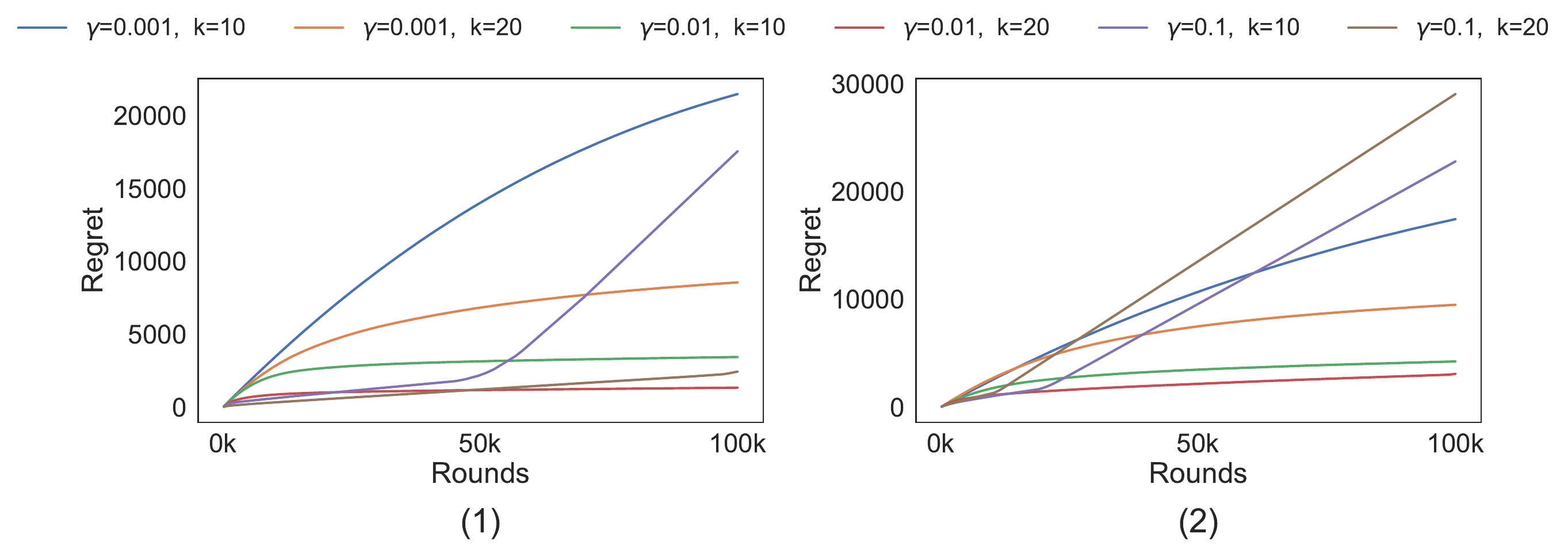}
    \caption{Ablation study of EXP3.M with varying mixing coefficient $\gamma \in \{0.001, 0.01, 0.1\}$.}
    \label{fig:gamma}
\end{figure}  

{\rmfamily
\begin{table}[H]
\caption{Ablation study of EXP3.M under different exploration parameters $\gamma$ on the synthetic dataset. 
Results are shown under two base-arm initial value distributions: (1) $\mathtt{U}(0,0.1)$; (2) $\mathtt{U}(0.3,0.4)$.}
\centering
\renewcommand{\arraystretch}{1.2}
\setlength{\tabcolsep}{1.5mm}
\resizebox{\textwidth}{!}{
\begin{tabular}{c c c c c c c c c c c c c c}
\toprule
\multirow{3}{*}{\textbf{Metric}} & \multirow{3}{*}{\textbf{Algorithm}} & \multicolumn{6}{c}{Range $[0,0.1]$} & \multicolumn{6}{c}{Range $[0.3,0.4]$} \\
\cmidrule(lr){3-8} \cmidrule(lr){9-14}
 & & \multicolumn{3}{c}{$k=10$} & \multicolumn{3}{c}{$k=20$} & \multicolumn{3}{c}{$k=10$} & \multicolumn{3}{c}{$k=20$} \\
\cmidrule(lr){3-5} \cmidrule(lr){6-8} \cmidrule(lr){9-11} \cmidrule(lr){12-14}
 & & $\gamma=0.001$ & $\gamma=0.01$ & $\gamma=0.1$  
 & $\gamma=0.001$ & $\gamma=0.01$ & $\gamma=0.1$
 & $\gamma=0.001$ & $\gamma=0.01$ & $\gamma=0.1$
 & $\gamma=0.001$ & $\gamma=0.01$ & $\gamma=0.1$ \\
\midrule

\multirow{1}{*}{Regret} 
& EXP3.M 
& 21475.680 & 3457.137 & 17551.325 
& 8570.484 & 1357.281 & 2459.607
& 17459.760 & 4255.882 & 22816.466
& 9514.021 & 3108.351 & 29068.987 \\
\midrule

\multirow{1}{*}{Runtime (sec)} 
& EXP3.M 
& 113.997 & 111.016 & 139.714
& 119.251 & 92.074 & 99.780
& 95.312 & 89.501 & 90.747
& 109.813 & 86.766 & 62.930 \\
\bottomrule
\end{tabular}
}
\label{tab:gamma}
\end{table}
}
\subsection{Cascading bandit}
\label{append:cascading}

Furthermore, we simulate a disjunctive cascading bandit problem with $m = 30$ base arms and cardinality $k = 10$ and $k = 20$.  In each round, the algorithm recommends a list of $k$ items, and the user examines them sequentially—either in descending or ascending order of attraction probability—clicking on the first attractive item (if any) and ignoring the remaining ones. If none of the items is deemed attractive, the round results in no  feedback and is treated as a no-op. The experimental configurations for both datasets (including synthetic and Yelp) and ablation studies in this section follow the same setup as detailed in \Cref{sec:experiments} and \Cref{append:experiments}. Particularly, all experiments presented hereafter adopt this disjunctive feedback model. Visual comparisons across all experiments are summarized in \Cref{fig:descending} (descending) and \Cref{fig:ascending} (ascending).

In addition, we \textbf{do not} include the configuration with $k = 20, m = 30$ under the $[0.3, 0.4]$ regime. This is because, under the cascading setting, the regret gap becomes extremely small in this regime: with such a large slate size relative to the attraction range, the user almost always clicks on at least one item with high probability, making the optimal expected reward very close to~1 (as implied by the Bernoulli click model). As a result, different algorithms exhibit nearly indistinguishable performance, and the comparison becomes uninformative.

\subsubsection{Descending bandit}
As shown in \Cref{tab:descending}, \textbf{CMOSS} consistently achieves lower cumulative regret than CUCB across both initialization ranges and datasets under the descending-order cascading model. For the $[0,0.1]$ initialization range, CMOSS reduces regret by \textbf{51.7\%} on the synthetic dataset and \textbf{39.6\%} on the Yelp dataset. For the $[0.3,0.4]$ range, the improvements remain notable at \textbf{15.1\%} and \textbf{25.1\%}, respectively. In terms of computational cost, CMOSS incurs moderate overhead in the $[0,0.1]$ regime, running only \textbf{1.14$\times$} slower than CUCB on the synthetic dataset and \textbf{1.03$\times$} slower on Yelp. Under the $[0.3,0.4]$ initialization, the runtime gap remains limited, with slowdowns of \textbf{1.51$\times$} on the synthetic dataset and \textbf{1.01$\times$} on Yelp. Overall, CMOSS delivers consistently improved regret performance while maintaining competitive computational efficiency, demonstrating robustness across both synthetic and real-world environments under descending-order feedback.

\begin{figure}[H]
    \centering
    \includegraphics[width=\textwidth]{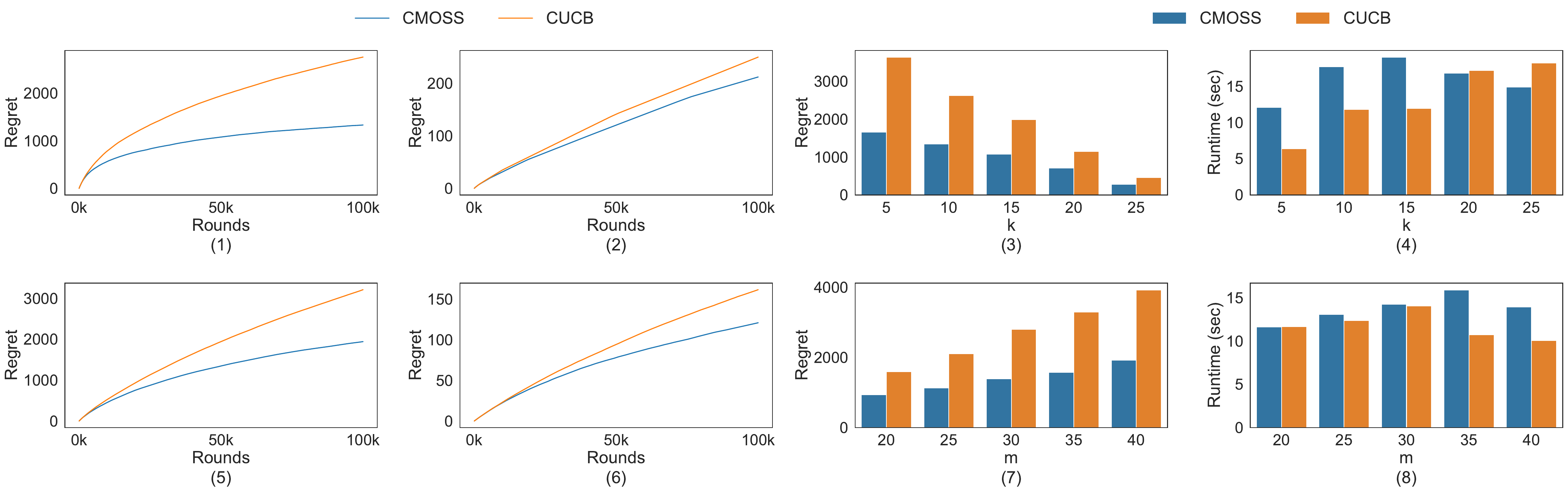}
    \caption{Comparison of CMOSS (blue) with CUCB algorithms under cascading (descending) feedback.  
Subplots (1)(2)(5)(6) show cumulative regret with fixed \(k=10\), \(m=30\); (3)(4) show ablation studies varying \(k\) (fixed $m=30$), while (7)(8) varying \(m\) (fixed $k=15$). Subplots (1)(2)(3)(4)(7)(8) use synthetic dataset; (5)(6) use Yelp dataset. Initial means of base arms fall within the range \([0, 0.1]\), except for (2)(6), which use the range \([0.3,0.4]\). }
    \label{fig:descending}
\end{figure}

{\rmfamily
\begin{table}[H]
\caption{Descending bandit feedback: cumulative regret and runtime (in seconds) of algorithms under two initial base arm value ranges: \([0,0.1]\) and \([0.3,0.4]\), evaluated on both synthetic and Yelp datasets and fixed $k=10, m=30$. 
Results correspond to subplots (1)(2)(5)(6) in \Cref{fig:descending}. \\
$\star$ indicates our proposed method (CMOSS).
}
\centering
\renewcommand{\arraystretch}{1.2}
\setlength{\tabcolsep}{1.5mm}
\begin{tabular}{c c c c c c}
\toprule
\multirow{2}{*}{\textbf{Metric}} & \multirow{2}{*}{\textbf{Algorithm}} & \multicolumn{2}{c}{Range $[0,0.1]$} & \multicolumn{2}{c}{Range $[0.3,0.4]$} \\
\cmidrule(lr){3-4} \cmidrule(lr){5-6}
 & & Synthetic & Yelp & Synthetic & Yelp \\
\midrule
\multirow{2}{*}{Regret} 
 & \textbf{CMOSS}$\star$  & \textbf{1333.863} & \textbf{1944.260} & \textbf{212.443} & \textbf{120.917}  \\
 & CUCB & 2762.497 & 3214.506 & 250.376 & 161.587  \\
\midrule
\multirow{2}{*}{Runtime (sec)} 
 & \textbf{CMOSS}$\star$ & \textbf{14.617} & \textbf{16.694} & \textbf{5.012} & \textbf{8.391}  \\
 & CUCB & 12.777 & 16.246 & 3.328 & 8.337 \\
\bottomrule
\end{tabular}
\label{tab:descending}
\end{table}
}

\Cref{tab:descending_ablation_k_m} and \Cref{tab:descending_ablation_k_m_yelp} summarize the ablation results on the synthetic and Yelp datasets under the descending-order cascading model within the range $[0,0.1]$, where we vary the cardinality $k$ and the number of base arms $m$. On the \textbf{synthetic} dataset, CMOSS consistently achieves the lowest cumulative regret across all configurations. When varying $k$, CMOSS outperforms CUCB by between \textbf{41.6\%} and \textbf{52.7\%} across all settings. When varying $m$, the improvements remain notable, ranging from \textbf{43.6\%} up to \textbf{49.8\%} relative to CUCB. In terms of runtime, CMOSS remains close to CUCB, with overheads between \textbf{1.05$\times$} and \textbf{1.31$\times$} across $k$ and \textbf{1.16$\times$} to \textbf{1.24$\times$} across $m$, while maintaining stable runtime as $m$ increases—demonstrating strong scalability. On the \textbf{Yelp} dataset, we observe consistent trends. When varying $k$, CMOSS achieves regret reductions ranging from \textbf{34.1\%} to \textbf{46.7\%} relative to CUCB. When varying $m$, CMOSS again demonstrates clear advantages, improving regret by \textbf{38.7\%} to \textbf{46.4\%} compared with CUCB. Runtime results further confirm CMOSS's efficiency: overheads range from \textbf{1.11$\times$} to \textbf{1.56$\times$} across $k$ and \textbf{1.15$\times$} to \textbf{1.27$\times$} across $m$. Overall, CMOSS maintains consistently low regret while preserving competitive runtime, highlighting its practical scalability under descending-order bandit feedback.

{\rmfamily
\begin{table}[H]
\caption{Ablation study of descending bandit feedback: varying $k$ (with $m=30$ fixed) and varying $m$ (with $k=15$ fixed) on synthetic dataset within the range $[0,0.1]$. Results correspond to subplots (3)(4)(7)(8) in \Cref{fig:descending}.\\  $\star$ indicates our proposed method (CMOSS).}
\centering
\renewcommand{\arraystretch}{1.2}
\setlength{\tabcolsep}{1.5mm}
\resizebox{\textwidth}{!}{
\begin{tabular}{c c c c c c c c c c c c}
\toprule
\multirow{2}{*}{\textbf{Metric}} & \multirow{2}{*}{\textbf{Algorithm}} & \multicolumn{5}{c}{Varying $k$ (Synthetic)} & \multicolumn{5}{c}{Varying $m$ (Synthetic)} \\
\cmidrule(lr){3-7} \cmidrule(lr){8-12}
  & & $k=5$ & $k=10$ & $k=15$ & $k=20$ & $k=25$ & $m=20$ & $m=25$ & $m=30$ & $m=35$ & $m=40$ \\
\midrule
\multirow{2}{*}{Regret} 
 & \textbf{CMOSS}$\star$ 
 & \textbf{2096.027} & \textbf{1783.914} & \textbf{1523.988} & \textbf{1376.970} & \textbf{661.014}
 & \textbf{516.490} & \textbf{680.074} & \textbf{993.672} & \textbf{1217.936} & \textbf{1449.846} \\
 & CUCB 
 & 4429.727 & 3446.924 & 3184.634 & 2578.753 & 1131.007
 & 916.646 & 1391.695 & 1952.693 & 2273.459 & 2751.108 \\
\midrule
\multirow{2}{*}{Runtime (sec)} 
 & \textbf{CMOSS}$\star$ 
 & \textbf{8.606} & \textbf{13.616} & \textbf{19.021} & \textbf{23.666} & \textbf{28.344}
 & \textbf{30.834} & \textbf{20.063} & \textbf{19.582} & \textbf{17.276} & \textbf{17.919} \\
 & CUCB 
 & 6.574 & 11.654 & 16.397 & 21.287 & 26.991
 & 24.877 & 16.767 & 16.434 & 14.984 & 14.915 \\
\bottomrule
\end{tabular}}
\label{tab:descending_ablation_k_m}
\end{table}
}

{\rmfamily
\begin{table}[H]
\caption{Ablation study of descending bandit feedback: varying $k$ (with $m=30$ fixed) and varying $m$ (with $k=15$ fixed) on Yelp dataset within the range $[0,0.1]$. \\ $\star$ indicates our proposed method (CMOSS).}
\centering
\renewcommand{\arraystretch}{1.2}
\setlength{\tabcolsep}{1.5mm}
\resizebox{\textwidth}{!}{
\begin{tabular}{c c c c c c c c c c c c}
\toprule
\multirow{2}{*}{\textbf{Metric}} & \multirow{2}{*}{\textbf{Algorithm}} &
\multicolumn{5}{c}{Varying $k$ (Yelp)} &
\multicolumn{5}{c}{Varying $m$ (Yelp)} \\
\cmidrule(lr){3-7} \cmidrule(lr){8-12}
 & & $k=5$ & $k=10$ & $k=15$ & $k=20$ & $k=25$ &
 $m=20$ & $m=25$ & $m=30$ & $m=35$ & $m=40$ \\
\midrule
\multirow{2}{*}{Regret}
 & \textbf{CMOSS}$\star$
   & \textbf{2132.156} & \textbf{1546.039} & \textbf{1257.476}
   & \textbf{856.376} & \textbf{501.748}
   & \textbf{413.413} & \textbf{724.900} & \textbf{1285.257}
   & \textbf{1507.817} & \textbf{2071.933} \\
 & CUCB
   & 3236.333 & 2602.221 & 2150.017
   & 1607.918 & 919.645
   & 674.154 & 1290.947 & 2385.077
   & 2659.719 & 3752.076 \\
\midrule
\multirow{2}{*}{Runtime (sec)}
 & \textbf{CMOSS}$\star$
   & \textbf{12.793} & \textbf{18.497} & \textbf{21.315}
   & \textbf{28.603} & \textbf{25.448}
   & \textbf{31.914} & \textbf{24.223} & \textbf{22.999}
   & \textbf{20.283} & \textbf{21.359} \\
 & CUCB
   & 8.215 & 13.148 & 16.720
   & 19.717 & 22.924
   & 25.212 & 20.810 & 17.956
   & 17.513 & 15.967 \\
\bottomrule
\end{tabular}
}
\label{tab:descending_ablation_k_m_yelp}
\end{table}
}

\subsubsection{Ascending bandit}
As shown in \Cref{tab:ascending}, \textbf{CMOSS} achieves consistently lower cumulative regret than CUCB across all value ranges and datasets under the ascending-order cascading model. Under the $[0,0.1]$ initialization, CMOSS reduces regret by \textbf{49.3\%} on the synthetic dataset and \textbf{45.2\%} on the Yelp dataset for $k=10$. Under the higher-value initialization $[0.3,0.4]$, CMOSS continues to outperform CUCB, achieving regret reductions of \textbf{49.9\%} on the synthetic dataset and \textbf{44.1\%} on Yelp for $k=10$. In terms of runtime, CMOSS incurs moderate overhead relative to CUCB, running approximately \textbf{1.25$\times$} slower on synthetic data and \textbf{0.85$\times$} on Yelp for $k=10$ under $[0,0.1]$ initialization, and \textbf{1.18$\times$} slower on synthetic data and \textbf{1.24$\times$} on Yelp under $[0.3,0.4]$ initialization. Overall, CMOSS provides a favorable trade-off between regret performance and computational cost, consistently improving cumulative regret while maintaining competitive efficiency.

\begin{figure}[H]
    \centering
    \includegraphics[width=\textwidth]{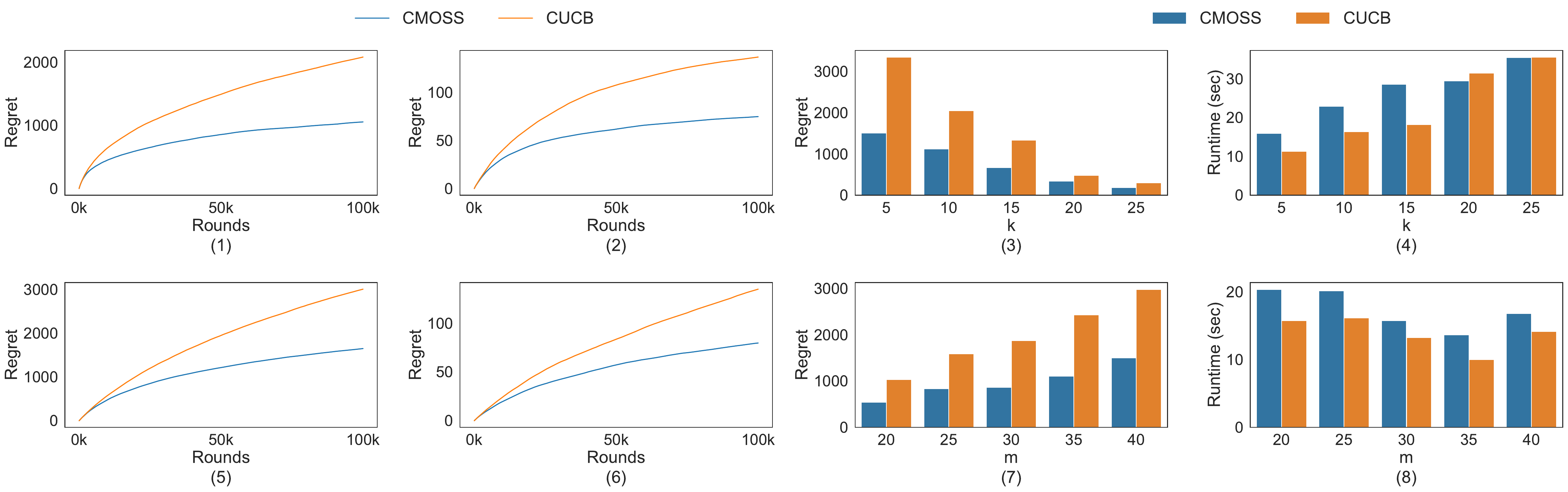}
    \caption{Comparison of CMOSS (blue) with CUCB algorithms under cascading (ascending) feedback.  
Subplots (1)(2)(5)(6) show cumulative regret with fixed \(k=10\), \(m=30\); (3)(4) show ablation studies varying \(k\) (fixed $m=30$), while (7)(8) varying \(m\) (fixed $k=15$). Subplots (1)(2)(3)(4)(7)(8) use synthetic dataset; (5)(6) use Yelp dataset. Initial means of base arms fall within the range \([0, 0.1]\), except for (2)(6), which use the range \([0.3,0.4]\). }
    \label{fig:ascending}
\end{figure}

{\rmfamily
\begin{table}[H]
\caption{Ascending bandit feedback: cumulative regret and runtime (in seconds) of algorithms under two initial base arm value ranges: \([0,0.1]\) and \([0.3,0.4]\), evaluated on both synthetic and Yelp datasets, with fixed $k=10, m=30$.  
Results correspond to subplots (1)(2)(5)(6) in \Cref{fig:ascending}. \\
$\star$ indicates our proposed method (CMOSS).
}
\centering
\renewcommand{\arraystretch}{1.2}
\setlength{\tabcolsep}{1.5mm}
\begin{tabular}{c c c c c c}
\toprule
\multirow{2}{*}{\textbf{Metric}} & \multirow{2}{*}{\textbf{Algorithm}} 
& \multicolumn{2}{c}{Range $[0,0.1]$} 
& \multicolumn{2}{c}{Range $[0.3,0.4]$} \\
\cmidrule(lr){3-4} \cmidrule(lr){5-6}
 & & Synthetic & Yelp & Synthetic & Yelp \\
\midrule
\multirow{2}{*}{Regret} 
 & \textbf{CMOSS}$\star$  & \textbf{1056.272} & \textbf{1651.691} & \textbf{1974.476} & \textbf{2265.129} \\
 & CUCB & 2083.085 & 3010.909 & 3945.550 & 4050.910 \\
\midrule
\multirow{2}{*}{Runtime (sec)} 
 & \textbf{CMOSS}$\star$ & \textbf{16.684} & \textbf{15.371} & \textbf{12.601} & \textbf{14.609} \\
 & CUCB & 13.395 & 18.039 & 10.724 & 11.739 \\
\bottomrule
\end{tabular}
\label{tab:ascending}
\end{table}
}

\Cref{tab:ascending_ablation_k_m} and \Cref{tab:ascending_ablation_k_m_yelp} report ablation results on the synthetic and Yelp datasets under the ascending-order cascading model, where we vary the cardinality $k$ and the number of base arms $m$. On the \textbf{synthetic} dataset, CMOSS consistently outperforms CUCB across all configurations. When varying $k$, CMOSS reduces cumulative regret by \textbf{41.6–52.7\%}, and when increasing $m$, it achieves reductions of \textbf{48.0–53.7\%}. In terms of runtime, CMOSS incurs moderate overhead for varying $k$, with relative costs ranging from \textbf{1.05$\times$} to \textbf{1.31$\times$}, while runtime is more stable when varying $m$, ranging between \textbf{1.13$\times$} and \textbf{1.29$\times$} relative to CUCB. On the \textbf{Yelp} dataset, similar trends hold. CMOSS consistently achieves lower regret, improving over CUCB by \textbf{43.4–45.1\%} when varying $k$ and \textbf{37.5–39.9\%} when varying $m$. Runtime remains competitive: for $k$, the overhead ranges from \textbf{0.92$\times$} to \textbf{1.51$\times$}, and for $m$, it ranges from \textbf{1.33$\times$} to \textbf{1.55$\times$}. Overall, CMOSS demonstrates robust regret reduction and stable runtime behavior, confirming its scalability under the ascending-order feedback model.

{\rmfamily
\begin{table}[H]
\caption{Ablation study of ascending bandit feedback: varying $k$ (with $m=30$ fixed) and varying $m$ (with $k=15$ fixed) on synthetic dataset. Results correspond to subplots (3)(4)(7)(8) in \Cref{fig:ascending}.\\  $\star$ indicates our proposed method (CMOSS).}
\centering
\renewcommand{\arraystretch}{1.2}
\setlength{\tabcolsep}{1.5mm}
\resizebox{\textwidth}{!}{
\begin{tabular}{c c c c c c c c c c c c}
\toprule
\multirow{2}{*}{\textbf{Metric}} & \multirow{2}{*}{\textbf{Algorithm}} & \multicolumn{5}{c}{Varying $k$ (Synthetic)} & \multicolumn{5}{c}{Varying $m$ (Synthetic)} \\
\cmidrule(lr){3-7} \cmidrule(lr){8-12}
  & & $k=5$ & $k=10$ & $k=15$ & $k=20$ & $k=25$ & $m=20$ & $m=25$ & $m=30$ & $m=35$ & $m=40$ \\
\midrule
\multirow{2}{*}{Regret} 
 & \textbf{CMOSS}$\star$
 & \textbf{2096.027} & \textbf{1783.914} & \textbf{1523.988} & \textbf{1376.970} & \textbf{661.014}
 & \textbf{226.889} & \textbf{352.963} & \textbf{485.097} & \textbf{577.284} & \textbf{812.244} \\
 & CUCB
 & 4429.727 & 3446.924 & 3184.634 & 2578.753 & 1131.007
 & 436.651 & 755.159 & 1041.633 & 1283.360 & 1752.840 \\
\midrule
\multirow{2}{*}{Runtime (sec)} 
 & \textbf{CMOSS}$\star$
 & \textbf{8.606} & \textbf{13.616} & \textbf{19.021} & \textbf{23.666} & \textbf{28.344}
 & \textbf{33.199} & \textbf{22.407} & \textbf{21.071} & \textbf{18.665} & \textbf{18.651} \\
 & CUCB
 & 6.574 & 11.654 & 16.397 & 21.287 & 26.991
 & 25.692 & 19.102 & 16.068 & 16.542 & 16.562 \\
\bottomrule
\end{tabular}}
\label{tab:ascending_ablation_k_m}
\end{table}
}

{\rmfamily
\begin{table}[H]
\caption{Ablation study of ascending bandit feedback: varying $k$ (with $m=30$ fixed) and varying $m$ (with $k=15$ fixed) on Yelp dataset. \\ $\star$ indicates our proposed method (CMOSS).}
\centering
\renewcommand{\arraystretch}{1.2}
\setlength{\tabcolsep}{1.5mm}
\resizebox{\textwidth}{!}{
\begin{tabular}{c c c c c c c c c c c c}
\toprule
\multirow{2}{*}{\textbf{Metric}} & \multirow{2}{*}{\textbf{Algorithm}} &
\multicolumn{5}{c}{Varying $k$ (Yelp)} &
\multicolumn{5}{c}{Varying $m$ (Yelp)} \\
\cmidrule(lr){3-7} \cmidrule(lr){8-12}
 & & $k=5$ & $k=10$ & $k=15$ & $k=20$ & $k=25$ &
 $m=20$ & $m=25$ & $m=30$ & $m=35$ & $m=40$ \\
\midrule
\multirow{2}{*}{Regret}
 & \textbf{CMOSS}$\star$
 & \textbf{2173.740} & \textbf{1337.191} & \textbf{816.458} & \textbf{524.540} & \textbf{227.607}
 & \textbf{295.952} & \textbf{653.558} & \textbf{936.492} & \textbf{1185.859} & \textbf{1618.286} \\
 & CUCB
 & 3958.232 & 2224.800 & 1618.416 & 915.447 & 402.428
 & 492.410 & 1059.454 & 1615.740 & 2240.731 & 2589.184 \\
\midrule
\multirow{2}{*}{Runtime (sec)}
 & \textbf{CMOSS}$\star$
 & \textbf{17.769} & \textbf{20.981} & \textbf{24.283} & \textbf{28.160} & \textbf{28.929}
 & \textbf{35.472} & \textbf{24.901} & \textbf{23.824} & \textbf{22.578} & \textbf{22.374} \\
 & CUCB
 & 11.738 & 13.694 & 21.308 & 29.020 & 31.540
 & 26.714 & 21.569 & 19.727 & 18.980 & 14.424 \\
\bottomrule
\end{tabular}
}
\label{tab:ascending_ablation_k_m_yelp}
\end{table}
}
Overall, these results demonstrate that CMOSS consistently achieves superior performance compared to CUCB across both descending- and ascending-order cascading bandit models. It significantly reduces cumulative regret across all datasets and configurations of cardinality \(k\) and number of base arms \(m\), while maintaining competitive runtime efficiency. Notably, CMOSS scales well with increasing \(m\), often running faster or with only moderate overhead compared to CUCB, especially in larger-scale settings. This robust performance across diverse scenarios highlights CMOSS’s practicality and reliability, making it a strong candidate for real-world recommendation and ranking applications where both accuracy and computational cost are critical.

\end{document}